\documentclass{article}
\usepackage[dblblindworkshop, final]{neurips_2025}
\workshoptitle{DiffCoALG: Differentiable Learning of Combinatorial Algorithms}
\usepackage[utf8]{inputenc}
\usepackage[T1]{fontenc}   
\usepackage{hyperref}     
\usepackage{url}           
\usepackage{booktabs}     
\usepackage{amsfonts}    
\usepackage{microtype}   

\usepackage{xcolor}       
\usepackage{todonotes}
\usepackage{anyfontsize}
\usepackage{amsthm}
\usepackage{thm-restate}
\usepackage{nicefrac} 
\usepackage{amsmath}
\usepackage{fontawesome5}
\usepackage{algorithm}
\usepackage{subcaption}
\usepackage{algpseudocode}
\usepackage{multirow}
\usepackage{graphicx}
\usepackage{longtable}
\newtheorem{theorem}{Theorem}

\newtheorem{corollary}{Corollary}
\newtheorem{definition}{Definition}

\title{Local Fragments, Global Gains: Subgraph Counting using Graph Neural Networks}

\author{%
  Shubhajit Roy \\
  Indian Institute of Technology Gandhinagar\\
  \texttt{royshubhajit@iitgn.ac.in} \\
   \And
   Shrutimoy Das\\
  Indian Institute of Technology Gandhinagar\\
  \texttt{shrutimoydas@iitgn.ac.in} \\
   \And
  Binita Maity\\
  Indian Institute of Technology Gandhinagar\\
  \texttt{binitamaity@iitgn.ac.in} \\
   \And
  Anant Kumar\\
  Indian Institute of Technology Gandhinagar\\
  \texttt{kumar\_anant@iitgn.ac.in} \\
   \And
  Anirban Dasgupta\\
  Indian Institute of Technology Gandhinagar\\
  \texttt{anirbandg@iitgn.ac.in} \\
}

\begin{document}

\maketitle

\begin{abstract}
Subgraph counting is a fundamental task for analyzing structural patterns in graph-structured data, with important applications in domains such as computational biology and social network analysis, where recurring motifs reveal functional and organizational properties. In this paper, we propose localized versions of the Weisfeiler-Leman (WL) algorithms to improve both expressivity and computational efficiency for this task. We introduce Local $k$-WL, which we prove to be more expressive than $k$-WL and at most as expressive as $(k+1)$-WL, and provide a characterization of patterns whose subgraph and induced subgraph counts are invariant under Local $k$-WL equivalence. To enhance scalability, we present two variants---Layer $k$-WL and Recursive $k$-WL---that achieve greater time and space efficiency compared to applying $k$-WL on the entire graph. Additionally, we propose a novel fragmentation technique that decomposes complex subgraphs into simpler subpatterns, enabling the exact count of all induced subgraphs of size at most $4$ using only $1$-WL, with extensions possible for larger patterns when $k>1$. Building on these ideas, we develop a three-stage differentiable learning framework that combines subpattern counts to compute counts of more complex motifs, bridging combinatorial algorithm design with machine learning approaches. We also compare the expressive power of Local $k$-WL with existing GNN hierarchies and demonstrate that, under bounded time complexity, our methods are more expressive than prior approaches.
\end{abstract}
\section{Introduction}

Graphs naturally model relational and structural data that arise across diverse domains, from social network analysis and combinatorial optimization to computational biology, particle physics, and protein folding \citep{dill2008protein}. Graph Neural Networks (GNNs) have emerged as powerful methods for learning graph representations for downstream tasks such as classification and link prediction. The expressive power of GNNs has been extensively studied, with message-passing GNNs shown to be equivalent to the $1$-dimensional Weisfeiler–Leman (WL) test \citep{morris2019wl}, a classical combinatorial tool for graph isomorphism. More generally, $k$-GNNs have been shown to correspond to the $k$-WL test \citep{weisfeiler1968reduction}, and thus, throughout this paper, we use $k$-WL interchangeably to denote the corresponding $k$-GNN hierarchy. In practice, the expressiveness of $k$-WL is determined by its ability to distinguish non-isomorphic graphs and capture subgraph patterns. However, while $(k+1)$-WL is provably more expressive than $k$-WL, both the time and space complexity of these algorithms grow exponentially in $k$, making higher-order WL infeasible in large-scale settings.

A central motivation driving this work lies in the task of \emph{subgraph counting}, which is central to understanding complex graph structures. Counting specific subgraphs provides insights into graph similarity and functionality: in computational biology, this includes motifs critical for protein binding and molecular discovery; in social networks, structures such as triangles and stars reveal community dynamics and information propagation. Unfortunately, detecting and counting subgraphs — a generalization of the clique problem — is NP-complete in general. Prior works have therefore focused on specific families of patterns or restricted graph classes \citep{shervashidze2009efficient, JMLR:v12:shervashidze11a, arvind2017graph, https://doi.org/10.48550/arxiv.1805.02089, https://doi.org/10.48550/arxiv.2006.09252, DBLP:conf/stacs/KomarathKMS23}. At the same time, research has sought to enrich GNN hierarchies, including $S_k$-GNNs \citep{papp2022theory}, which augment nodes with counts of induced subgraphs of size up to $k$, and $M_k$-GNNs \citep{papp2022theory, huang2023boosting}, which modify graphs via marked or deleted vertices.

Traditional GNNs, being limited to the power of $1$-WL, cannot count even simple motifs such as triangles or cycles beyond length three \citep{xu2019powerfulgraphneuralnetworks, morris2021weisfeilerlemanneuralhigherorder}. Subgraph-based GNN models \citep{zhao2022from, zhang2021nested, frasca} have emerged as extensions that apply GNNs on extracted local subgraphs or multiple partitions of the input graph, thereby improving expressiveness. While these approaches capture more fine-grained structures, they suffer from scalability concerns since GNNs must be run repeatedly over many subgraphs, creating prohibitive memory and time overhead.

To overcome these limitations, we propose \emph{localized} variants of Weisfeiler–Leman algorithms that balance expressiveness and computational efficiency. Instead of running $k$-WL across the entire graph, we restrict its application to neighborhoods, resulting in the \emph{Local $k$-WL}. Complementary to this, we introduce a novel \textbf{fragmentation} technique, which decomposes complex motifs into simpler subpatterns whose counts can be reliably obtained. Remarkably, we prove that all induced subgraphs of size at most 4 can be counted exactly using only $1$-WL, with natural extensions to larger motifs for higher $k$.

In addition, we design a differentiable three-stage learning framework that operationalizes this fragmentation principle: (1) identify required subpatterns, (2) count them using Local $k$-WL, and (3) aggregate counts into the global motif count. This approach provides both theoretical guarantees and compatibility with gradient-based learning, bridging combinatorial counting algorithms with machine learning paradigms. Finally, we analyze the expressive power of our methods relative to existing GNN hierarchies, demonstrating that under bounded complexity, our models achieve strictly higher expressivity.

\subsection{Our Contributions}

We summarize our key contributions:
\begin{enumerate}
    \item \textbf{Expressiveness of Local $k$-WL:} We formally introduce Local $k$-WL, which applies $k$-WL to $r$-hop neighborhoods $G_v^r$ rooted at $v \in V.$ We provide, for the first time, both upper and lower bounds for its expressiveness and characterize precisely the subgraph patterns that can be distinguished and counted.
    
    \item \textbf{Layer $k$-WL:} A scalable variant where $k$-WL is restricted to pairs of consecutive BFS layers, reducing time and space overhead compared to standard Local $k$-WL.
    
    \item \textbf{Recursive WL:} A hierarchical alternative to $k$-WL. We first apply $1$-WL to partition nodes, and then run $(k-1)$-WL on smaller vertex sets. This achieves higher expressiveness than $(k-1)$-WL while maintaining better scalability than $k$-WL.
    
    \item \textbf{Fragmentation for Subgraph Counting:} We propose a decomposition technique that reduces the counting of complex motifs to that of simpler subpatterns. For example, we reduce $K_4$ counting to repeated triangle counting. We prove that all induced subgraphs of size at most 4 can be counted with $1$-WL and extend this methodology to larger patterns.
    
    \item \textbf{Comparative Expressiveness:} We provide a rigorous comparison of Local $k$-WL, Layer $k$-WL, and Recursive WL against existing GNN hierarchies such as $S_k$ and $M_k$. Our results show that, under comparable complexity bounds, our models are strictly more expressive.
\end{enumerate}

In summary, our work advances subgraph counting in GNNs by combining the theoretical rigor of WL with scalable localized techniques, while introducing a new fragmentation paradigm that makes counting feasible, differentiable, and efficient for real-world graph applications.

\paragraph{Outline of the paper :} Section \ref{prelim} introduces some of the terms used throughout the paper. 
In Section \ref{various gnn}, we introduce the localized variants of the $k-$WL algorithm and analyze their space and time complexities. Section \ref{exp_power_gnn} gives theorems that characterize the expressiveness of the localized $k-$WL variants proposed in our work. In Section \ref{subgraph_counting_sec}, we characterize the expressiveness of our methods in terms of subgraph and induced subgraph counting. We also discuss how to count the occurrences of $H$ in $G,$ using localized algorithms. We discuss the fragmentation approach in Section \ref{fragmentation_sec}, followed by a theoretical comparison of GNN models in Section \ref{sec:compare_mk_local}. The model architecture in Section \ref{sec:model}. Discusson on implementation details, hyperparameters and results of our experiments is in Section \ref{sec:experiments} and conclude the paper with Section \ref{conclusion}.

\section{Preliminaries}
\label{prelim}

We consider a simple undirected graph $G = (V,E)$ with vertex set $V$ and edge set $E$. For basic graph-theoretic definitions, we refer the reader to \citep{west2001introduction}. The neighbourhood of a vertex $v \in V$ is the set of all adjacent vertices, denoted by $N_G(v)$ (or $\mathcal{N}_G(v)$). The \emph{closed neighbourhood}, denoted $N_G[v]$, includes $v$ itself. The degree of a node $v$ is denoted by $d_v$. The maximum distance from $v$ to any other vertex is its \emph{eccentricity}, and the minimum eccentricity over all vertices gives the \emph{radius} of the graph. A graph in which every vertex has the same degree is called a \emph{regular graph}.  

A graph $H$ is a \emph{subgraph} of $G$ if $V(H)\subseteq V(G)$ and $E(H)\subseteq E(G)$. The subgraph \emph{induced} on $S \subseteq V(G)$, denoted $G[S]$, contains vertex set $S$ with all edges in $G$ whose endpoints lie in $S$. The induced subgraph on an $r$-hop neighbourhood of a vertex $v$ is denoted by $G_v^r$. The hop parameter $r$ depends on the pattern being counted. For example, $r=1$ suffices for counting triangles, while $r=2$ is required for counting $C_4$. Attributed subgraphs (or \emph{motifs}) are subgraphs whose vertices or edges carry additional labels or colours.  

A \emph{graph homomorphism} from $H$ to $G$ is a function $f: V(H) \to V(G)$ such that $\{u,v\}\in E(H)$ implies $\{f(u),f(v)\}\in E(G)$. The set of all homomorphic images of a pattern $H$ is called the \emph{spasm} of $H$. Two graphs $G$ and $H$ are \emph{isomorphic} if there exists a bijection $f: V(G)\to V(H)$ preserving adjacency. The \emph{orbit} of a vertex $v$ in $G$, denoted $Orbit_G(v)$, is the set of vertices to which $v$ can be mapped under automorphisms of $G$.  

\subsection{Graph Parameters}
Many hard graph problems become tractable on restricted graph classes characterized by structural parameters. We briefly review key notions.

\textbf{Treewidth.} A \emph{tree decomposition} of $G$ expresses it as a tree of vertex “bags” satisfying: (i) every vertex of $G$ appears in some bag, (ii) every edge is contained in some bag, and (iii) for each vertex $v$, bags containing $v$ form a connected subtree. The \emph{width} of a decomposition is the maximum bag size minus one, and the \emph{treewidth} $tw(G)$ is the minimum width across all decompositions. Computing $tw(G)$ is NP-hard, but fixed-parameter algorithms exist \citep{korhonen2022single, korhonen2022improved}. The \emph{hereditary treewidth} $htw(H)$ of a pattern $H$ is the maximum treewidth among its homomorphic images.

\textbf{Planar graphs and genus.} A graph is \emph{planar} if it can be drawn without crossing edges, or equivalently, if it excludes $K_5$ and $K_{3,3}$ as minors. More generally, the \emph{Euler genus} denotes the minimal surface on which a graph can be embedded without edge crossings.

\textbf{Clique-width and rank-width.} For dense graphs, \emph{clique-width} provides a measure of complexity, but recognizing graphs of clique-width at most $k\geq 4$ is NP-hard. \emph{Rank-width}, introduced by Robertson and Seymour, offers an alternative based on rank decompositions of the adjacency matrix. It is known that bounded clique-width implies bounded rank-width, and vice versa \citep{oum2017rank}.

\subsection{Weisfeiler–Leman and Local $k$-WL}

The Weisfeiler–Leman (WL) test \citep{weisfeiler1968reduction} is a colour refinement algorithm for testing graph isomorphism. Its $k$-dimensional extension, $k$-WL, refines colours of $k$-tuples of vertices, with expressiveness increasing in $k$. Applying $k$-WL to a full graph with $n$ vertices requires $O(n^{k+1}\log n)$ time. In contrast, the \emph{Local $k$-WL} algorithm applies $k$-WL to restricted neighbourhoods $G_v^r$, reducing complexity. For graphs of maximum degree $d$, local $k$-WL can be run in $O(n \cdot d^{r(k+1)} \log d)$ time.  

\subsection{Graph Neural Networks}

Graph Neural Networks (GNNs) \citep{DBLP:journals/corr/KipfW16} generalize neural architectures to graphs using recursive \texttt{MESSAGE} and \texttt{AGGREGATE} functions. At the $\ell$-th iteration, node embeddings are updated as:
\begin{equation}
    X_v^{(\ell)} = \texttt{UPDATE}^{(\ell)} \left( 
    X_v^{(\ell-1)}, \;
    \texttt{AGGREGATE}^{(\ell)} \Big\{ 
    \texttt{MESSAGE}^{(\ell)} \big(X_v^{(\ell-1)}, \, X_u^{(\ell-1)}\big) 
    \;\big|\; u \in N_G(v) \Big\}\right).
\end{equation}
It is known that standard message-passing GNNs are bounded in expressiveness by $1$-WL, motivating higher-order and subgraph-based extensions.
\section{Weisfeiler Leman Algorithm}
\label{WL_SECTION}
Weisfeiler-Leman (WL) is a well-known combinatorial algorithm that has many theoretical and practical applications. Color refinement($1-$WL) was first introduced in 1965 in \citep{morgan1965generation}. The algorithm goes as follows:

\begin{itemize}
    \item Initially, color all the vertices as color 1.
    \item In the next iteration $i,$ we color the vertices by looking at the number of colors of vertices adjacent to each vertex $v,$ in the $(i-1)$th iteration, as $$C_i(v)=(C_{i-1}(v),\{\{C_i(w)\}\}_{w\in N_G(v)}) $$
    We assign a new color to the vertices, according to the unique tuple it belongs to. This partitions the vertex set in every iteration according to their color classes.
    \item The algorithm terminates if there is no further partition. We call the color set a \emph{stable} color set.
    \item We can also observe that if two vertices get different colors at any stage $i,$ then they will never get the same color in the later iterations. We can observe that the number of iterations is at most $n$ as a vertex set ,$V(G),$ can be partitioned at most $n$ many times. 
    \item The color class of any vertex $v\in V(G)$ can appear at most $1+\log n$ times and the running time is $\mathcal{O}(n^2\log n)$ \citep{immerman2019kdimensional}.
\end{itemize}

In case we need to run only $h$ iterations and stop before getting the stable color, then the running time is $O(nh)$. 

The same idea was extended by Weisfeiler and Leman in which instead of coloring vertex, they colored all the two tuples based on edge, non-edge and $(v,v)$. In later iteration, the color gets refined for each two tuples based on their neighbourhood and common neighbourhood. This partition the set of two tuples of vertices. The iteration in which no further partition is being done are called \emph{stable coloring}. Weisfeiler Leman algorithm which is known as $2-$WL algorithm.

Similar approach was extended later for coloring $k-$tuples and then do refinement of coloring in later iterations. 

\begin{definition}
Let $\vec{x}= (x_1, \dots , x_k)\in V^k$,$y\in V$ , and $1\leq j\leq k$. Then, let $x[j, y]\in V^k$ denote the $k-$tuple obtained from $x$ by replacing $x_j$ by $y$. The $k-$tuples $\vec{x}[j, y]$ and $\vec{x}$ are said to be $j-$neighbors for any $y\in V$. We also say $\vec{x}[j, y]$ is the j-neighbor of $\vec{x}$ corresponding to $y$.
\end{definition}
\begin{itemize}
    \item Color all the $k-$tuple vertices according to their isomorphic type. Formally, $(v_1,v_2,....,v_k)$ and $(w_1,w_2,....,w_k)$ get the same color if $v_i=v_j$ then $w_i=w_j$ and also, if $(v_i,v_j)\in E(G),$ then $(w_i,w_j)\in E(G)$.
    \item In every iteration, the algorithm updates the color of the tuple after seeing the color of its adjacent $k$ tuple vertices.
    $$C_{i+1}^k(\vec{v}) :=(C_i^k(\vec{v},M_i(\vec{v}) $$ where $M_i(\vec{v})$ is the multiset $$\{\{(C_i^k(v_1,v_2,...,v_{k1},w),...,C_i^k(v_1,v_2,..,w,..,v_k),...,C_i^k(w,v_2,...,v_k))\mid w\in V\}\}$$
    \label{k-dimension update}
    \item The algorithm terminates if there is no further partition. We call the color set a stable color set.
    \item We also observe that if two tuples get different colors at any stage $i,$ then they will never get the same color in the later iterations. We can observe that the number of iterations is at most $n^k$ as $V^k$ can be partitioned at most $n^k$ many times. 
    \item The color class of any vertex $\vec{v}\in V^k$ can appear at most $\mathcal{O}(k\log n)$ times and running time is $\mathcal{O}(k^2n^{k+1}\log n)$ \citep{immerman2019kdimensional}.
\end{itemize}

Two graphs $G$ and $H$ are said to be $k-$WL equivalent ($G\simeq_k H$), if their color histogram of the stable colors matches. We say that $G$ is $k-$WL identifiable if there doesn't exist any non-isomorphic graphs that are $k-$WL equivalent to $G$.

 Color refinement ($1-$WL) can recognise almost all graphs \citep{babai1980random}, while $2-$WL can recognise almost all regular graphs \citep{bollobas1982distinguishing}. The power of $WL$ increases with an increase in the value of $k$. The power of $k-$WL to distinguish two given graphs is same as with counting logic $C^{k+1}$ with $(k+1)-$variable. Also, the power of $k-$WL to distinguish two non-isomorphic graphs is equivalent to spoiler's winning condition in $(k+1)-$bijective pebble game. Recently, \citep{Dell2018LovszMW} has shown that the expressive power of $k-$WL is captured by homomorphism count. It has been shown that $G_1\simeq_{k} G_2 $ if and only if $Hom(T,G_1) = Hom (T,G_2),$ for all graphs $T$ of treewidth at most $k$.

The graphs that are identified by $1-$WL are \emph{Amenable} graphs. There is a complete characterization of the amenable graphs in \citep{arvind2017graph,kiefer2015graphs}. In the original algorithm, we initially colour all the vertices with colour $1$. However, if we are given a coloured graph as input, we start with the given colours as the initial colours. Also, we can color the edges and run $1-$WL \citep{kiefer2015graphs}.

Even if ${k-WL}$ may not distinguish two non-isomorphic graphs, two ${k-WL}$ equivalent graphs have many invariant properties. It is well known that two ${1-WL}$ equivalent graphs have the same maximum eigenvalue. Two graphs that are ${2-WL}$ equivalent are co-spectral and have the same diameter. Recently, V. Arvind et al. have shown the invariance in terms of subgraph existence and counts \citep{arvind2020weisfeiler}. They show the complete characterization of subgraphs whose count and existence are invariant for ${1-WL}$ equivalent graph pairs. They also listed down matching, cycles, and path count invariance for ${2-WL}$. Also, there is a relation between homomorphism count and subgraph count \citep{curticapean2017homomorphisms}. The count of subgraphs is a function of the number of homomorphisms from a set of all homomorphic images of patterns. \emph{Hereditary treewidth} of a graph is defined as the maximum treewidth over all homomorphic images. So, if two graphs are ${k-WL}$ equivalent, then the count of all subgraphs whose $htw(G)$ is almost $k$ is the same. However, running ${k-WL}$ takes $O(k^2 \cdot n^{k+1}logn )$ time and $O(n^k)$ space \citep{immerman2019kdimensional}. So, it is not practically feasible to run ${k-WL}$ for large $k$, for graphs.

The expressive power of $k-$WL is equivalent to first-order logic on $(k+1)$ variables with a counting quantifier. Let $G=(V,E)$ where $V$ is a set of vertices and $E$ is a set of edges. In logic, we define $V$ as the universe and $E$ as a binary relation. In \citep{cai1992optimal}, they have proved that the power to distinguish two non-isomorphic graphs using $k-$WL is equivalent to $C^{k+1}$, where $C^{k+1}$ represents first-order logic on $(k+1)$ variables with counting quantifiers (stated in \ref{pebble-wl-relation}). To prove this, they define a bijective $k-$pebble game, whose power is equivalent to $C^{k}$. 

\textbf{Bijective k-Pebble Game}
\label{PEBBLE_BIJECTIVE}

The bijective k-Pebble game ($BP_{k}(G, H)$) has been discussed in \citep{kiefer2020power,cai1992optimal,rankwidth_WL}.
Let graphs $G$ and $H$ have the same number of vertices and $k\in \mathbb{N}$. Let $v_i,v \in V(G)$ and $w_i,w \in V(H)$. 
\begin{definition}
    The position of the game in a bijective pebble game is the tuples of the vertices where the pebbles are placed.
\end{definition}

The bijective $k-$pebble game is defined as follows:
\begin{enumerate}
    \item Spoiler and Duplicator are two players.
    \item Initially, no pebbles are placed on the graphs. So, the position of the game is ((),()) (the pairs of empty tuples.)
    \item The game proceeds in the following rounds as follows:
    \begin{enumerate}
        \item Let the position of the game after the $i^{th}$ round be $((v_1,...,v_l),(w_1,w_2,...,w_l))$. Now, the Spoiler has two options: either to play a pebble or remove a pebble. If the Spoiler wants to remove a pebble, then the number of pebbles on the graph must be at least one and if Spoiler decides to play a pebble then number of pebbles on that particular graph must be less than $k$.
        
        \item If the Spoiler wants to remove a pebble from $v_i,$ then the current position of the game will be $((v_1,v_2,...v_{i-1},v_{i+1},..,v_l),(w_1,w_2,...w_{i-1},w_{i+1},..,w_l))$. Note that, in this round, the Duplicator has no role to play.
        
        \item If the Spoiler wants to play a pair of pebbles, then the Duplicator has to propose a bijection $f:V(G)\xrightarrow{} V(H)$ that preserves the previous pebbled vertices. Later, the Spoiler chooses $v\in V(G)$ and sets $w=f(v)$. The new position of the game is $((v_1,...v_l,v),(w_1,w_2,...,w_l,w))$.
    \end{enumerate}
\end{enumerate}
The Spoiler wins the game if for the current position $((v_1,...v_l,v),(w_1,w_2,...,w_l,w)),$ the induced graphs are not isomorphic. If the game never ends, then the Duplicator wins. The equivalence between the bijective $k-$pebble game and $k-$WL was shown in the following theorem.
\begin{theorem}\citep{cai1992optimal}
\label{pebble-wl-relation}
Let $G$ and $H$ be two graphs. Then $G\simeq_k H$ if and only if the Duplicator wins the pebble game $BP_{k+1}(G, H).$
\end{theorem}

A stronger result, namely, the equivalence between the number of rounds in the bijective $(k+1)-$pebble game and the iteration number of $k-$WL was stated in the following theorem.
\begin{theorem}\citep{kiefer2020power}
\label{Iteration_number_logic}
Let $G$ and $H$ be graphs of the same size. The vertices may or may not be coloured. Let $\vec{u} := (u_{1}, \dots , u_{k}) \in (V (G))^{k}$ and $\vec{v} := (v_{1}, \dots , v_{k}) \in
(V (H))^{k}$ be any two arbitrary elements. Then, for all $i \in \mathbb{N}$, the following are equivalent :
\begin{enumerate}
    \item The color of $\vec{u}$ is same as the color of $\vec{v}$ after running $i$ iterations of $k-$WL.
    \item For every counting logic formulae with $(k+1)$ variables of quantifier depth at most $i$, $G$ holds the formula if and only if $H$ does so.
    \item Spoiler does not win the game $BP_{k+1}(G, H)$ with the initial configuration $(\vec{u}, \vec{v})$ after at most $ i$ moves.
\end{enumerate}
\end{theorem}
\section{Local k-WL based Algorithms for GNNs}\label{various gnn}

This section presents the local $k-$WL based algorithms for GNNs. We also give runtime and space requirements for such GNNs.

\subsection{Local k-WL}
Given a graph $G$, we extract the subgraph induced on a $r-$hop neighbourhood around every vertex. We refer to it as $G_v^r,$ for the subgraph rooted at vertex $v$ in $G$. Then, we colour the vertices in $G_v^r$ according to their distances from $v$. Now, we run $k-$WL on the coloured subgraph $G_v^r$. The stable colour obtained after running $k-$WL is taken as the attributes of vertex $v$. Then, we run a GNN on the graph $G$ with the attributes on each vertex $v$. This is described in \ref{Local k-WL}.
\begin{algorithm}[ht]
\caption{Local k-WL}
\label{Local k-WL}
\begin{algorithmic}[1]
\State Input: $G,r,k$
\For{each vertex $v$ in $V(G)$} \label{alg:line2}
        \State Find the subgraph induced on the $r-$hop neighborhood rooted at vertex $v$ ($G_v^r$). 
        \State Color the vertices whose distance from $v$ is $i,$ by color $i$.
        \State Run $k-$WL on the colored graph until the colors stabilize.\label{alg1:line5}
    \EndFor
\State Each vertex has as an attribute the stable colouring of vertex $v$ obtained from $G_v^r$.

\State Run GNN on the graph $G$ with each vertex having attributes as computed above.
\end{algorithmic}
\end{algorithm}

\textbf{Runtime and Space requirement Analysis : }
The time required to run $k-$WL on $n$ vertices is $O(n^{k+1}\log(n))$. Here, we run $k-$WL on a $r-$hop neighbourhood instead of the entire graph. So, $n$ is replaced by $n_1,$ where $n_1$ is the size of the neighbourhood. If a graph has bounded degree $d$, and we run $k-$WL for a $2-$hop neighbourhood, then $n_1$ is $O(d^2)$. Also, we have to run Local $k-$WL for each vertex. Hence, the total time required is $O(n\cdot d^{2k+2}\log(d))$. Also, running a traditional GNN takes time $O((n+m)\log n),$ where $m$ is the number of edges. So, if we assume that $d$ is bounded, then the time required is linear in the size of the graph. Furthermore, the space required to run $k-$WL on $n$ vertices graph is $O(n^k)$. Hence, for Local $k-$WL, it follows that the space requirement is $O(n_1^k)$.
\subsection{Layer k-WL}
In order to make Local $k-$WL more time and space-efficient while maintaining the same expressive power, we propose a modification to Local $k-$WL. Instead of running $k-$WL on the entire $r-$hop neighbourhood, we run $k-$WL on consecutive layers of $G_{v}^{r}$ (i.e., run $k-$WL on the set of vertices with colour $i$ and colour $(i+1)$). Initially, we run $k-$WL on the set of vertices that are at a distance of $1$ and $2$ from $v.$ Then, we run $k-$WL on the set of vertices with colors $2$ and $3,$ and so on. While running $k-$WL, initially, it partitions the $k-$tuples based on the isomorphism type. However, we incorporate the stabilized colouring obtained in the previous round in this setting. For $l < k,$ we define the color of $l$ tuples as 
$col(u_1,u_2,...,u_l):=col(u_1,u_2,...,u_l,\underbrace {u_1,u_1..,,u_1}_{(k-l) \text{times}})$. Consider the mixed tuple (we call a tuple to be mixed if some of the vertices have been processed in the previous iteration and the remaining have not yet been processed) $(u_1,v_1,\ldots, u_k)$ where $col(u_j)=i$ and $col(v_j)=i+1$ (i.e $u_i'$s are the set of processed vertices and $v_i'$s are yet to be processed). So, even if $(u_1,v_1,\ldots , u_k)$ and $(u_1^{'},v_1^{'},\ldots , u_k^{'})$ may be isomorphic, if $col(u_1,u_2,\ldots u_l)\neq col(u_1^{'},u_2^{'},\ldots u_l^{'})$ then $col(u_1,v_1,\ldots , u_k) \neq col (u_1^{'},v_1^{'},\ldots , u_k^{'})$. The algorithm is described in \ref{Layer k-WL}. A GNN model incorporating Local+Layer $k-$WL is equivalent to running Layer $k-$WL in line \ref{alg1:line5} in \ref{Local k-WL}.

 \begin{algorithm}[ht]
\caption{Layer k-WL($v$)\label{Layer k-WL}}

\begin{algorithmic}[1]
\State Given $G_v^r,k$.
\State Run $k-$WL on the induced subgraph of levels $1$ and $2$.
\For{each layer $i$ of BFS($v$), $i\ge 2$}
    \State Initial colour of $k-tuple$ incorporates the stabilized colour obtained from the previous iteration.
    \State Run $k-$WL on the subgraph induced on the vertices in layer $i$ and $(i+1)$    
\EndFor
\end{algorithmic}
\end{algorithm}

\textbf{Runtime and Space requirement Analysis}.

The running time and space requirement for Layer $k-$WL depends on the maximum number of vertices in any two consecutive layers, say $n_2$. The time required to run $k-$WL is $O(r\cdot (n_2)^{k+1}\log(n_2))$. However, running only Local $k-$WL will require $O((r\cdot n_2)^{k+1}\log(r\cdot n_2))$ time. The space requirement is $O(n_2^k).$
Hence, running Layer $k-$WL is more efficient than running Local $k-$WL, especially when $r$ is large.

\subsection{Recursive WL}
Here, we present another variant of WL. The central idea is to decompose the graphs initially by running $1-$WL. Then, further, decompose the graphs by running $2-$WL and so on. One can note that the final vertex partition that $1-$WL outputs after colour refinement is regular if we restrict it to a single colour class. In other words, let $G[X]$ be the induced graph on the vertices of the same colour. Then, $G[X]$ is regular. Also, $G[X,Y]$ where $X$ and $Y$ are sets of vertices of two different colour classes. $G[X,Y]$ is also a bi-regular graph. We run $2-$WL on the regular graph. Using \citep{bollobas1982distinguishing}, we can guarantee that it would distinguish almost all regular graphs. Similarly, running $2-$WL on $G[X,Y]$ is bi-regular and thus can be distinguished by $2-$WL. 
We again run $1-$WL on $G,$ using the colours obtained after running $2-$WL. This further refines the colours of the vertices in $G$. One can easily check that it is more expressive than $1-$WL and less expressive than $2-$WL. We give the graph \ref{Graph identifiable by 2-WL but not by Recursive $1-$WL} that can not be distinguished by Recursive $(1,2)-$WL and $1-$WL but can be distinguished by $2-$WL. This gives an intermediate hierarchy in the $k-$WL hierarchy. Also, the space and time required for running $2-$WL on the entire graph is more than that of Recursive $(1,2)-$WL. The running time and space required depend on the partition size obtained after running $1-$WL. 

Note that the colour of vertex $v$ is $col(v,v)$ after running $2-$WL. 
 \begin{algorithm}[ht]
\caption{Recursive(1,2) WL\label{algo:recursive}}
\label{alg:recursive}
\begin{algorithmic}[1]
\State Given $G$
\State Run $1-$WL and get the partition of vertices into colour classes.
\State Let $S=\{C_1,C_2,\ldots C_l\}$ be the color classes obtained after running $1-$WL.
\For{each color class $C_i$ in $S$} \label{algrec:line2}
        \State Run $2-$WL on the induced subgraph in $C_i$ and get color partition.
        \State Let $C_i$ get partitioned into $C_{i,1}, C_{i,2},\ldots,C_{i,l}$
    \EndFor
\State Run $1-$WL on the colored graph $G$ whose colors are given by steps 5 and 6.     
\For{each new color class $C'_{i}$ and $C'_{j}$}
    \State Run $2-$WL on the induced subgraph on the vertices in color partitions $C'_i$ and $C'_j$ and get new color partition.
\EndFor
    
\State Repeat 5-11 till the colour stabilizes.
\end{algorithmic}
\end{algorithm}

This idea can be generalized for any suitable $k.$ We can run a smaller dimensional $k_1-$WL and then use the partition of $k_1$ tuple vertices. Later, we can use this partition to get a finer partition of $k_2$ tuples. Assuming $k_1< k_2$, one can see that we have to run $k_2-$WL on smaller graphs. This reduces the time and space required for running $k_2-$WL on the entire graph. One can easily see that it is less expressive than $k_2-$WL; however, it is more expressive than $k_1-$WL. 

More specifically, initially run $1-$WL and further run $(k-1)-$WL on the coloured classes. One can check that it is more expressive than $(k-1)-$WL and less expressive than $k-$WL. 

\begin{figure}[ht]
    \centering
    \captionsetup{font=footnotesize}
    \includegraphics[scale=0.5]{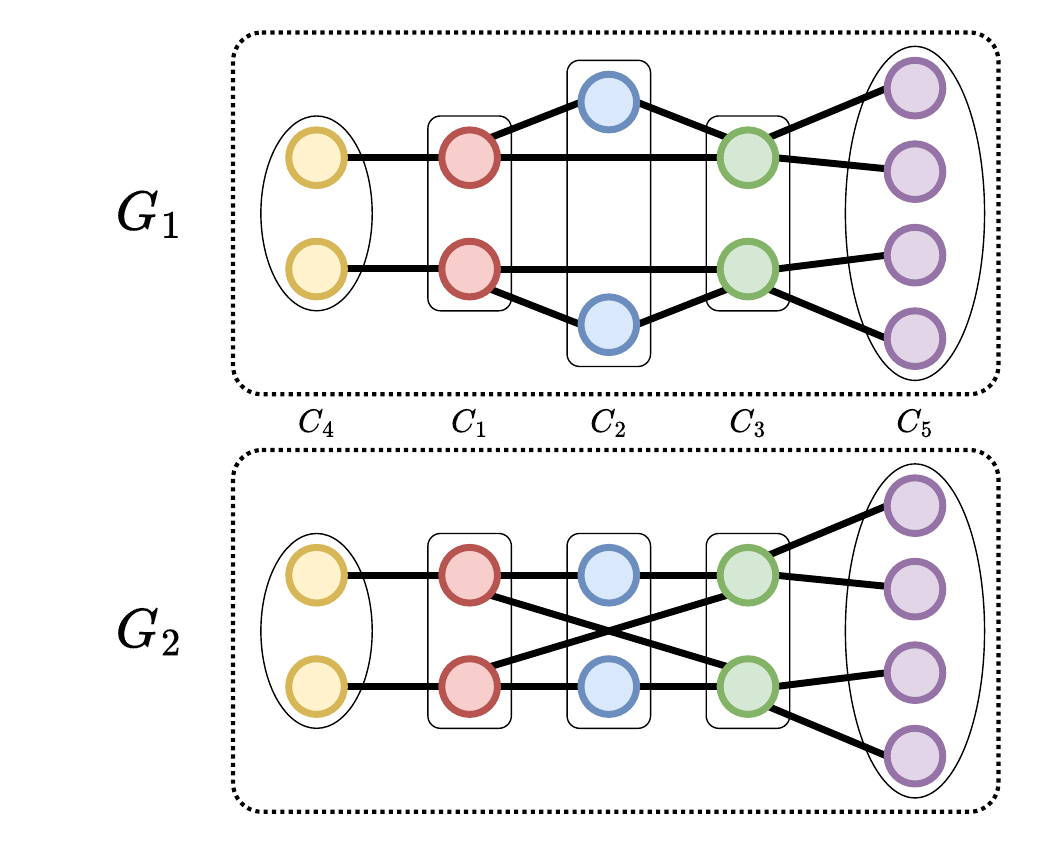}
    \caption{\small Graphs identifiable by 2-WL but not by Recursive $1-$WL}
    \label{Graph identifiable by 2-WL but not by Recursive $1-$WL} 
\end{figure}
\section{Theoretical Guarantee of Expressive Power}\label{exp_power_gnn}
In this section, we theoretically prove the expressive power of GNN models that we proposed in \ref{various gnn} in terms of graph and subgraph isomorphism. In the discussion below, we say that a GNN model $A$ is at most as expressive as a GNN model $B$ if any pair of non-isomorphic graphs $G$ and $H$ that can be distinguished by $A$ can also be distinguished by $B$. Also, we say a GNN model $A$ is at least as expressive as a GNN model $B$ if $A$ can identify all the non-isomorphic graph pairs that can be identified by GNN model $B$. The proof of the theorem and lemmas presented in the section mainly use pebble bijective game. Also, as mentioned earlier, there is equivalence between the expressivity of $k-$WL and $(k+1)-$pebble bijective game.
\subsection{Local k-WL}
It has been shown in recent works that running Local $1-$WL has more expressive power as compared to running $1-$WL. The upper bound on the expressive power of Local $1-$WL has been shown in \citep{frasca}. However, the expressive power of Local $k-$WL, for arbitrary $k,$ has not been studied. In the Theorem \ref{local_wl_power}, we give a description of the expressive power of Local $k-$WL and show that it has more expressive power than $k-$WL. We also show that it is less expressive than $(k+1)-$WL. The proof techniques that we used are different from \citep{frasca}.
\begin{restatable}{theorem}{localwl}
    \label{local_wl_power}
 Running Local $k-$WL is more expressive than running $k-$WL on the entire graph. Also, running Local $k-$WL is at most as expressive as running $(k+1)-$WL on the entire graph.

\end{restatable}
\begin{proof}\label{Proof:local-k-wl power}
Let $G_1$ and $G_2$ be two graphs distinguishable by $k-$WL. So, the Spoiler has a winning strategy in the game ($G_1$,$G_2$). Suppose $G_1$ and $G_2$ do not get distinguished after running $k-$WL locally. That means the Duplicator has a winning strategy for all vertices individualized. Let $v$ in $G_1$ and $u$ in $G_2$ be the vertices that are individualized.

We play the $(k+1)-$bijective pebble game on the entire graphs $(G_1,G_2)$ and the local subgraphs $(G_1^v, G_2^u)$ simultaneously. Let $S_1$ and $D_1$ be the spoiler and duplicator in game $(G_1,G_2)$ respectively, and $S_2$ and $D_2$ be the spoiler and duplicator in game $(G_1^v,G_2^u)$. We use the strategy of $D_2$ to determine the move for $D_1$ and the strategy of $S_1$ to determine the move for $S_2$.

Initially, $D_2$ gives a bijection $f$ from the vertex set of $G_1^v$ to $G_2^u$. We propose the same bijection $f$ by $D_1,$ extending it by mapping $v$ to $u$. Now, the spoiler $S_1$ places a pebble at some vertex $(v_i,f(v_i)).$ The spoiler $S_2$ also places a pebble at vertex $(v_i,f(v_i))$. We can show using induction on the number of rounds that if $S_1$ wins the game, then $S_2$ also wins the game. Our induction hypothesis is that $S_1$ has not won till the $j^{th}$ round, and the positions of both games are the same. Let the current position of both the games after the $j^{th}$ round be $((v_1,v_2,\ldots,v_l),(f(v_1),f(v_2),\ldots,f(v_l))$. Now, $S_1$ either decides to play a pair of pebbles or remove. 

\textit{Case(1): If $S_1$ decides to remove a pebble.}
In this case, the Duplicator $D_1$ has done nothing to do. $S_2$ will copy the same strategy as $S_1$. Here, $S_1$ cannot win in this round. Also, note that the positions of both games are the same. 

\textit{Case(2): If $S_1$ decides to play a pebble.} In this case, $S_2$ also decides to play a pebble. Duplicator $D_2$ proposes a bijective function $f.$ The same bijective function is proposed by $D_1$. Now, $S_1$ places a pebble at $(v,f(v))$. $S_2$ also chooses the same vertices. So, the position of both the games is the same. Therefore, if $S_1$ wins the game, then $S_2$ also wins the game. Thus, running $k-$WL locally is at least as expressive as running $k-$WL on the entire graph.

We can show that it is more expressive by looking at the simple example that running $1-$WL on a local substructure can count the number of triangles, whereas running $1-$WL on an entire graph does not recognize graphs having different triangle counts. Also, one can observe that running $k-$WL locally is running $k-$WL on the coloured graph where vertices at distinct distances get distinct colours. Its power is the same as individualizing one vertex and running $k-$WL. Thus, running $k-$WL locally is more expressive than running $k-$WL on the entire graph.

Let $G_1$ and $G_2$ be two graphs that can be distinguished by running $k-$WL locally. Recall that the key vertices refer to $u$ and $v$ in $G_1$ and $G_2$ such that they are the root vertices corresponding to $G_1$ and $G_2$, respectively. This means the Spoiler has a winning strategy in the $(k+1)$ bijective pebble game, where the key vertices are matched to each other. Now, we use the strategy of the Spoiler in the local substructure to get a winning strategy for the Spoiler in the entire graph. At first, when the Duplicator gives a bijective function, the Spoiler places a pebble on the paired vertices. For the remaining moves, we copy the strategy of the Spoiler in the local structure, and the Duplicator's strategy of the entire graph is copied to the Duplicator's strategy of the local structures. Thus, if the Spoiler has a winning strategy in the local substructure, then the Spoiler wins the $(k+2)-$ bijective pebble game on entire graphs.
\end{proof}

\subsection{Layer k-WL}
We presented an algorithm (Algorithm \ref{Layer k-WL}) for applying $k-$WL to consecutive layers in a $r-$hop subgraph for a vertex $v \in V.$ This improves the time and space efficiency of the Local $k-$WL method as we have discussed above. We now describe the expressive power of Layer $k-$WL. In the following lemmas, we show that the expressive power of Layer $k-$WL is the same as that of Local $k-$WL.

\begin{restatable}{lemma}{lone}
    \label{layer-k-wl atleast}
    Running $k-$WL on the entire $r-$hop neighbourhood is at least as expressive as running Layer $k-$WL. 
\end{restatable} 

\begin{proof}
    Let $G$ and $H$ be the subgraphs induced on the $r-$hop neighbourhood. Let $(S, D)$ be the Spoiler-Duplicator pair for the game $(G, H)$. Similarly, let $(S_i,D_i)$ be the Spoiler-Duplicator pair for the game $(G_i,H_i),$ where $G_i$ and $H_i$ are the subgraphs induced on the vertices at the $i$th and $(i+1)$th layers of $G$ and $H,$ respectively. We claim that if any of the $S_i'$s has a winning strategy in the game $(G_i, H_i)$, then $S$ has a winning strategy in the game $(G, H)$. Here, the strategy of $D$ is copied by $D_i$, and the strategy of $S_i$ is copied by $S$. We prove this using induction on the number of rounds of the game. Our induction hypothesis is that the positions of both the games are the same, and if $S_i$ wins after $t$ rounds, then $S$ also wins after $t$ rounds. 

    \textit{Base case:} $D$ proposes a bijective function $f:V(G)\xrightarrow{} V(H)$. Note that the bijection must be colour-preserving; otherwise, $S$ wins in the first round. Thus, we can assume that $f$ is color-preserving. So, $D_i$ proposes the restricted function $f_i$ as a bijective function from $V(G_i)$ to $V(H_i)$. Now, $S_i$ plays a pair of pebbles in $(G_i, H_i)$, and $S$ also plays the same pair of pebbles in the game $(G, H)$. It is easy to see that the positions of the two games are the same. Also, if $S_i$ wins, then the number of vertices of a particular colour is different. Hence, $S$ also has a winning strategy. 

    By the induction hypothesis, assume that after the $t^{th}$ round $S_i$ did not win and the position of the game is the same for both games. 

    Consider the $(t+1)^{th}$ round in both games. $S_i$ either chooses to play or remove a pebble. If $S_i$ chooses to remove a pebble, so will $S$. Again, the position of both the games is the same. If $S_i$ decides to play a pair of pebbles, then $S$ also decides to play a pair of pebbles. So, $D$ proposes a bijective function, and $D_i$ proposes a restricted bijective function. Now, if $S_i$ plays a pair of pebbles at $(v,f_i(v))$, then $S$ also decides to play a pair of pebbles at $(v,f(v))$. Thus, the position of the game is the same in both of the games. This ensures that if $S_i$ has won, then $S$ also wins.
    
\end{proof}
\begin{restatable}{lemma}{ltwo}
    \label{layer-k-wl atmost}
    Running Layer $k-$WL is as expressive as running $k-$WL on the entire induced subgraph. 
\end{restatable}
\begin{proof}
    Let $G$ and $H$ be the subgraphs induced on a $r-$hop neighbourhood. Let $(S,D)$ be the Spoiler-Duplicator pair for the game $(G,H)$. Similarly, let $(S_i,D_i)$ be the Spoiler-Duplicator pair for the game $(G_i,H_i),$ where $G_i$ and $H_i$ are the subgraphs induced on the vertices at the $i$th and $(i+1)$th layers of $G$ and $H,$ respectively. We claim that if $S$ has a winning strategy in the game $(G,H),$ then there exists $S_i$ such that it has a winning strategy in the game $(G_i,H_i).$ Here, the strategy of $D$ is copied by $D_i$ and the strategy of $S_i$ is copied by $S.$ We prove the lemma using induction on the number of rounds of the game. Our induction hypothesis is that the position of the game $(G,H)$ is the same for $(G_i,H_i),$ for all $i,$ if we restrict it to the subgraph induced by the vertices of color $i$ and $(i+1).$ Also, if $S$ wins after round $t,$ then there exists $S_i$ that wins after $t$ rounds. 

    \textit{Base case:} $D$ proposes a bijective function $f:V(G)\longrightarrow V(H)$. Note that the bijection must be colour-preserving; otherwise, $S$ wins in the first round. Thus, we can assume that $f$ is color-preserving. So, $D_i$ proposes the restricted function $f_i$ as a bijective function from $V(G_i)$ to $V(H_i), \forall i\in [r]$. Now, $S$ will play a pair of pebbles in the game $(G,H).$ Suppose $S$ plays the pebbles at $(v,f(v))$ and $color(v)=i,$ then $S_i$ and $S_{i-1}$ play pebbles at $(v,f_i(v))$ in their first round. It is easy to see that the position of the games $(G,H)$ and $(G_i,H_i),$ for all $i\in[r],$ is the same if we restrict it to the subgraph induced by vertices of colours $i$ and $(i+1)$. Also, if $S$ wins, then the number of vertices of a particular colour is not the same. So, there exists some $i,$ such that $S_i$ also has a winning strategy. 

    By induction hypothesis, assume that after the $t^{th}$ round, $S$ did not win, and the position of the game is the same as defined. 

    Consider the $(t+1)^{th}$ round in both the games. $S$ either chooses to play or remove a pebble. If $S$ chooses to remove a pebble from vertex $(v,f(v)),$ then, if $v$ is colored with color $i,$ then $S_i$ and $S_{i-1}$ will remove a pebble from vertex $(v,f_i(v))$. Again, the position of both the games is the same. Now, if $S$ decides to play a pair of pebbles, then each $S_i$ also decides to play a pair of pebbles. So, $D$ propose a bijective function, and $D_i$ proposes a restricted bijective function. Now, suppose $S$ plays a pair of pebbles at $(v_1,f(v_1))$. If $color(v_1)=i,$ then $S_i$ and $S_{i-1}$ also decides to play pebbles at $(v_1,f_i(v_1))$. Thus, the position of the game is the same as defined. Now, if $S$ wins, then there exists $u$ and $v$ such that either $(u,v)\in E(G)$ and $(f(u),f(v))\notin E(H)$ or $(u,v)\notin E(G)$ and $(f(u),f(v))\in E(H)$. Similarly, there exists $S_i$ for which these things happen as the position induced is the same. Therefore, $S_i$ wins for some $i.$
\end{proof}

Thus, from the above two lemmas, we can say that the expressive power of Layer $k-$WL is the same as local $k-$WL.

\section{Subgraph Counting Algorithms and Characterization of Patterns}\label{subgraph_counting_sec}
Here, we characterize the expressive power of the proposed methods in terms of subgraph as well as induced subgraph counting. In this section, we provide algorithms and characterization of subgraphs that can exactly count the number of patterns appearing as subgraphs or induced subgraphs. As described above, we can see that the running time is dependent on the size of the local substructure and the value of $k$. The size of the subgraph is dependent on the radius of the patterns. So, we have to take a $r-$hop neighbourhood for each vertex $v$ in the host graph $G$.

In \ref{k_based on local substructue}, we show how the value of $k$ can be decided based on the local substructure of the host graph. It is independent of the structure of the pattern. Also, it gives an upper bound on the value of $k$ that can count patterns appearing as subgraphs and induced subgraphs. In \ref{subsection_induced subgraph}, we first show that doing local count is sufficient for induced subgraph count, and later, we give the upper bound on $k$ based on the pattern size. Note that the value of $k$ for induced subgraph count is based only on the size of the pattern, not its structure. In \ref{subsection_subgrpah count}, we again show that locally counting subgraph is sufficient. Also, we explore the value of $k$ based on the structure of the pattern. For subgraph counting, the structure of the pattern can be explored to get a better upper bound for the value of $k$. Later, for the sake of completeness, we give an algorithm to count triangles, the pattern of radius one and $r.$

\subsection{Deciding k based on local substructure of host graph} \label{k_based on local substructue}
Here, we explore the local substructure of the host graph in which we are counting patterns appearing as graphs and subgraphs. For a given pattern of radius $r$, we explore the $r-$hop neighbourhood around every vertex $v$ in the host graph $G$. If two graphs $G_1$ and $G_2$ are isomorphic, then the number of subgraphs and induced subgraphs of both the graphs are the same. We use the same idea to count the number of subgraphs.

\cite{cai1992optimal}, shown that $\Omega(n)$ dimension is needed to guarantee graph isomorphism. However, for restricted graph classes, we can still guarantee isomorphism for small dimensions. It has been shown that ${3-WL}$ is sufficient for planar graph \citep{kiefer2019weisfeiler}, ${k-WL}$ for graphs with treewidth at most $k$ \citep{decompose_graphs}, $(3k+4)-WL$ for graphs with rankwidth at most $k$ \citep{rankwidth_WL}, and $(4k+3)$ for graphs with Euler genus at most $k$ \citep{grohe2019linear}. We say these graph classes as \emph{good} graph classes. Note that, for non-isomorphic graphs, the graphs is not $k-$WL equivalent. Thus, running corresponding $k-$WL can count the pattern of radius $r$, appearing as a subgraph and induced subgraph.
\begin{restatable}{theorem}{thtwo}
    \label{isomorphism}
    Let $G_v^r$ denote the $r-$hop neighborhood around $v$. Given a pattern of radius $r,$ the values of $k$ that are sufficient to guarantee the count of patterns appearing either as subgraphs or induced subgraphs are:
    \begin{enumerate}
        \item $(3-WL)$ if $G_v^r$ planar
        \item $(k-WL)$ if $tw\footnote{tw denotes treewidth}(G_v^r)\leq k$
        \item $((3k+4)-WL)$ if $rankwidth(G_v^r)\leq k$
        \item $((4k+3)-WL)$ if $Euler-genus(G_v^r)\leq k$
    \end{enumerate}
\end{restatable}
\begin{proof}\label{Proof:Isomorphism}
Consider the subgraph induced by the vertex $v$ and its $r-$hop neighbourhood in $G,$ say $G_v^r,$ and the subgraph induced by the vertex $u$ and its $r-$hop neighbourhood in $H,$ say $H_u^r.$ Suppose both structures belong to \emph{good} graph classes. Now, we run corresponding $k$ based on the local substructure as mentioned in the theorem. If the colour histogram of the stable colour matches for both graphs. This implies that both the graphs are isomorphic. Thus, the number of subgraphs and induced subgraphs in both of the substructures are also the same.

Also, we run respective $k-$WL on a coloured graph, where vertices at a distance $i$ from $v$ are coloured $i$. So, it is at least as expressive as running $k-$WL on an uncoloured graph. We can also show that it is strictly more expressive in distinguishing non-isomorphic graphs. Thus, all the $k-$WL mentioned corresponding to \emph{good} graph classes are sufficient for counting the number of patterns appearing as subgraphs and induced subgraphs.
\end{proof}
\begin{corollary}
\label{Corollary 1}
  If $G_v^r$ is amenable, for all $v\in V(G),$ then Local $1-$WL outputs the exact count of the patterns appearing as subgraph and induced subgraph.
\end{corollary}
\begin{corollary}
\label{Corollary 2}
    Running $1-$WL guarantees the exact number of subgraphs and induced subgraphs of all patterns of radius one, when the maximum degree of the host graph is bounded by 5.

    Similarly, if the maximum degree of the host graph is bounded by $15$, then running $2-$WL is sufficient to count subgraphs and induced subgraphs of all patterns with a dominating vertex.
\end{corollary}
\subsection{Counting Induced Subgraphs}\label{subsection_induced subgraph}
The following lemma shows that we can easily aggregate the local counts of the pattern $H$ appearing as an induced subgraph to get the count of $H$ over the entire graph.
\begin{restatable}{lemma}{lthree}
\label{local_induced}
\small
\begin{equation}
	IndCount(H,G)= \frac{\sum_{v\in V(G)} IndCount_{(u,v)}(H,G_v^r)}{|Orbit_H(u)|}
	\end{equation}
\end{restatable}
\begin{proof}
    Suppose a pattern $H$ in $G$ appears as an induced subgraph. So, an injective homomorphism from $V(H)$ to $V(G)$ exists, such that it preserves the edges. We fix one subgraph and find the number of mappings possible. Suppose one of the mappings maps $u_i$ to $v_j$ where $j\in |V(H)|$. Now, we can see that the number of mappings of $u_i$(key vertex) is the same as the size of the orbit of $u_i$ in $H.$ This proves the claim that every induced subgraph has been counted, the size of the orbit many times.
\end{proof}

Assume that we want to count the number of occurrences of pattern $H$ in $G$ as (an induced) subgraph. Let $u$ be the key vertex of $H$ and $r$ be the radius of $H$. 

\begin{restatable}{lemma}{lfour}
  \label{r_hop_neighborhood}
    It is sufficient to look at the $r-$hop neighborhood of $v_i$ to compute $Count_{(u,v)}(H,G)$ or $IndCount_{(u,v)}(H,G)$. 
\end{restatable}

\begin{proof}
    Suppose a subgraph exists in $G$ that is isomorphic to $H$ or isomorphic to some supergraph of $H$ with an equal number of vertices, where $u$ is mapped to $v_i$. The homomorphism between the graphs preserves edge relations. Consider the shortest path between $u$ and any arbitrary vertex $u_i$ in $H$. The edges are preserved in the homomorphic image of $H$. Therefore, the shortest distance from $f(u)$ to any vertex $f(u_i)$ in $G$ is less than equal to $r$. So, it is sufficient to look at the $r-$hop neighborhood of $v_i$ in $G$.
\end{proof}

From the above two lemmas, we can conclude the following theorem:
\begin{restatable}{theorem}{ththree}
   \label{r_hop_neighbour}
    It is sufficient to compute $IndCount_{(u,v)}(H,G_v^r),$ for $i\in [n],$ where $G_v^r$ is induced subgraph of $r-$hop neighborhood vector of $v$. 
\end{restatable}

The following theorem gives a direct comparison with the $S_k$ model \citep{papp2022theory}, where each node has an attribute which is the count of induced subgraphs of size at most $k.$
\begin{restatable}{theorem}{thfour}
    Local $k-$WL can count all induced subgraphs of size at most $(k+2)$.
\label{induced_subgraph_general}
\end{restatable}
\begin{proof}\label{Proof:induced_subgraph_proof}
Suppose, if possible $G$ and $H$ are Local $k-$WL equivalent and $|P|\leq (k+2),$ where $P$ is the pattern to be counted. We will choose one vertex $v$ as a key vertex. Now, we want to count $P-v$ locally. Assume that the distance from $v$ to any vertex is $'r'$. So, we take the $r-$hop neighborhood of every vertex in $G$ and $H,$ respectively. It is easy to see that the number of induced subgraphs or subgraphs of size $k$ is the same locally if they are $k-$WL equivalent since we do an initial coloring of $k-$tuples based on isomorphism type. Now, suppose $P'=P-v$ is of size $(k+1)$. 

Let $P_i= P'-v_i,$ for $i\in[k+1].$ It is possible that there may exist two pairs of subgraphs that are isomorphic. In that case, we remove such graphs. Let $P_1, P_2,\ldots, P_l$ be all pairwise non-isomorphic graphs. $k-$WL would initially color the vertices according to isomorphism type. So, the subgraph count of $P_i$ is the same in any two $k-$WL equivalent graphs. Let $V(P_i)=(u_1,u_2,\ldots u_k)$. We see the refined color after one iteration in equation \ref{k-dimension update}.
Now, we can observe that by looking at the first coordinate of the color tuples in a multiset, we can determine the adjacency of $u$ with $(u_2,u_3,\ldots u_k)$. Similarly, after seeing the second coordinate of the color tuples in the multiset, we can determine the adjacency of $u$ with $(u_1,u_2,\ldots u_k)$.

Consider, $\forall {u}\in V(G),$ $P_i\cup \{u\}=H$ will give the count of induced subgraph $P'$. Thus, if $G$ and $H$ are $k-$WL equivalent, then the size of each color class after the first iteration will be the same. 

Now, for each $P'$ with $v$ will form $P$ if it has exactly $|N_P(v)|$ many vertices of color $1$. Also, as mentioned earlier that $k-$WL is equivalent to $C^{k+1}$ logic, and we have to add unary logic stating that the color of neighbor to be $1$. The $k-$WL and $C^{k+1}$ are equivalent, so we have to add the unary relation putting the condition of the required colors. Again, using \ref{local_induced} we can say that running $k-$WL locally can count all patterns of size $(k+2)$ appearing as an induced subgraph.
\end{proof}

We can see the corollary below in which we mention the set of patterns shown in \citep{chen2020can}.

\begin{corollary} 
\label{Corollary 3}
Local $2-$WL on the subgraphs induced by neighborhood of each vertex can count each pattern appearing as induced subgraphs as well as subgraphs of (a) 3-star (b) triangle (c) tailed triangle (d) chordal cycle (e) attributed triangle.
\end{corollary}

Based on the above results, we now present the algorithm \ref{algo:count_induced_subgraph} for counting patterns appearing as an induced subgraph in $G$ using the localized algorithm. The function $IndCount_{u,v}(H, G_v^r)$ takes as input the pattern $H$, the attributed version of $G_v^r$, and returns the induced subgraph count of $H$ in $G_v^r,$ where $u \in H$ is mapped to $v \in G_v^r$. Notice that the function $IndCount_{u,v}(H, G_v^r)$ is a predictor that we learn using training data. 

\begin{algorithm}[ht]
\caption{Counting induced subgraph H in G\label{algo:count_induced_subgraph}}
\label{Induced subgraph Count}
\begin{algorithmic}[1]
\State Given $G,H$.
\State Find $r=radius(H)$ and let $u\in H$ be a corresponding center vertex.  
\For{each vertex $v$ in V(G)}
    \State Extract subgraph $G_v^r$.
    \State Find suitable $k$, which will give an exact count based on the local substructure.
    \State Run Local+Layer $k-$WL on $G_v^r$. 
    \State Calculate $IndCount_{u,v}(H,G_v^r)$.
\EndFor
\State return $\frac{\sum_{v\in V(G)} IndCount_{(u,v)}(H,G_v^r)}{|Orbit_H(u)|}$
\end{algorithmic}
\end{algorithm}

The running time and space requirement for Algorithm \ref{algo:count_induced_subgraph} is dependent on the value of $k$ and $r$. 
We can make informed choices for the values of $k$ and $r.$ Notice that the value of $k$ is chosen based on the local substructure. Also, the value of $r$ is the radius of $H.$
Suppose the local substructure is simple (planar, bounded treewidth, bounded rankwidth \ref{isomorphism}). In that case, $k-$WL, for small values of $k,$ is sufficient for counting induced subgraph $H.$ Otherwise, we have to run $(|H|-2)-$WL in the worst case.
%
\subsection{Deciding k based on the pattern for counting subgraphs}\label{subsection_subgrpah count}
For any pattern $H$, it turns out that the number of subgraph isomorphisms from $H$ to a host graph $G$ is simply a linear combination of all possible graph homomorphisms from $H'$ to $G$ ($Hom(H', G)$ is the number of homomorphisms from $H'$ to $G$) where $H'$ is the set of all homomorphic images of $H$. That is, there exists constants $\alpha_{H'} \in \mathbb{Q}$ such that:
\begin{equation}\label{eq:subcount}
Count(H,G) = \sum_{H'} \alpha_{H'} Hom(H',G)
\end{equation}
where $H'$ ranges over all graphs in $H'$. This equation has been used to count subgraphs by many authors (Please refer \citep{Alon1997, curticapean2017homomorphisms}).

\begin{restatable}{theorem}{thfive}\citep{cai1992optimal,Dell2018LovszMW}
    \label{wl_invariant}
For all $k\geq 1$ and for all graphs $G$ and $H$, the following are equivalent: 
\begin{enumerate}
    \item $HOM(F,G) = HOM(F,H)$ for all graph $F$ such that $tw(F)\leq k$.
    \item $k-$WL does not distinguish $G$ and $H$ and 
    \item Graph $G$ and $H$ are $C^{k+1}$ equivalent\footnote{Counting logic with (k+1) variables}.
\end{enumerate}
\end{restatable}

Using \ref{eq:subcount} and \ref{wl_invariant}, we arrive at the following theorem:
\begin{restatable}{theorem}{thsix}
    \label{spasm_pattern}
 Let $G_1$ and $G_2$ be $k-$WL equivalent and $htw(H)\leq k$. Then subgraph count of $H$ in $G_1$ and $G_2$ are the same.
\end{restatable}

\begin{restatable}{lemma}{lfive}
    \label{local_subgraph}
\small
\begin{equation} \label{eqn2}
	Count(H,G)= \frac{\sum_{v\in V(G)} Count_{(u,v)}(H,G_v^r)}{|Orbit_H(u)|}
	\end{equation}
\end{restatable}
\begin{proof}\label{proof:local_subgraph}
    Suppose $H$ appears as subgraph in $G$. Then, there must exist an injective function matching $u\in V(H)$ to some $v\in V(G)$. Thus, counting locally, we can easily see that every subgraph would be counted. Now to prove \ref{eqn2}, it is sufficient to show that for a given subgraph it is counted exactly $|Orbit_H(u)|$ many times. Note that two subgraphs are same if and only if their vertex sets and edge sets are same. We fix the vertex set and edge set in $G$, which isomorphic to $H$. Now, consider an automorphism of $H$ which maps $u$ to one of the vertices $u'$ in its orbit. Note that we can easily find the updated isomorphic function that maps $u'$ to $v$. Now, the number of choices of such $u'$ is exactly $|Orbit_H(u)|$. Thus, the same subgraph is counted at least $|Orbit_H(u)|$ many times. Suppose $x\in V(H)$ is a vertex such that $x\notin Orbit_H(u)$. Considering the fixed vertex set and edge set, if we can find an isomorphism, then it is a contradiction to the assumption that $x\notin Orbit_H(u)$. Thus, the same subgraph is counted exactly $|Orbit_H(u)|$ many times.
\end{proof}

Using \ref{spasm_pattern} and \ref{local_subgraph}, one can easily see that for counting pattern $H$ as a subgraph, it is sufficient to run Local $k-$WL on the local substructure and count the subgraph locally.
\begin{restatable}{theorem}{thseven}
      Local $k-$WL can exactly count any subgraph $H$ if $htw(H-v)\leq k$.
\end{restatable}

The upper bound on the choice of $k$ for running $k-$WL can be improved from the default $|H| -2$ bound that we used for the induced subgraph count. The value of $k$ is now upper bound by $htw(H)$.
Hence, we pick the minimum $k$ based on the local substructure of $G$ as well as the hereditary treewidth of pattern $H$ for computing the subgraph count of $H$ in $G$. The algorithm for counting the subgraph is similar to the induced subgraph.

\begin{corollary}
\label{1-wl_subgraph}
    Local $1-$WL can exactly count the number of patterns $P$ or $P-v$ appearing as a subgraph, when $P$ or $P-v,$ with a dominating vertex $v$, is $K_{1,s}$ and $2K_2.$
\end{corollary}
\begin{proof}
   In a star, $K_{1,s},$ all the leaves are mutually independent. By the definition of homomorphism, the edges are preserved in homomorphic images. So, the only possibility of a homomorphic image of the star is getting another star with less number of leaves. Note that the star is a tree, and its treewidth is one. Also, for $2K_2$, the homomorphic image is either itself or the star. So, the treewidth of all its homomorphic images is $1.$
\end{proof}
\begin{corollary}
\label{c_4 subgraph count}
    Local $1-$WL can exactly count the number of $C_4$ appearing as a subgraph.
\end{corollary}
\begin{proof}
    We can see that $H-v$ is $K_{1,2}$ when choosing any vertex as a key vertex. Also, the orbit size is $4$. So, we can directly use Lemma 5 to compute the count of the $C_4$ in the graph locally and then sum it over all the vertices and divide it by $4$.
\end{proof}
\begin{corollary}
\label{neural_comparison_wl}
Local $1-$WL can exactly count patterns appearing as subgraphs of (a) 3-star (b) triangle (c) tailed triangle (d) chordal cycle (e)attributed triangle and patterns appearing as induced subgraphs of (a) triangle and (c) attributed triangle.
\end{corollary}
\begin{proof}
For subgraph counting, we can see that for all of the $5$ patterns, there exists a vertex $v$ such that $htw(P-v)=1.$ One can note that the attributed triangle can also be handled using Corollary 1. Since every pattern has a dominating vertex, running $1-$WL on the subgraph induced on the neighbourhood is sufficient.
Now, we only have to argue for the patterns appearing as induced subgraphs. Note that the induced subgraph count of the triangle and the attributed triangle is same as the subgraph count of the triangle and attributed triangle. 
\end{proof}

Note that all the subgraph or induced subgraph counting can be easily extended to attributed subgraph or attributed-induced subgraph counting (graph motif). We will be given a coloured graph as an input, and we will incorporate those colours and apply a similar technique as described above to get the subgraph count. 
\begin{restatable}{corollary}{cone}
    \label{attribute_patterns}
     If $C(G)=C(H),$ where $C(.)$ is the color histogram, then $Count(P,G)=Count(P,H)$ where $P$ is the attributed subgraph.
\end{restatable}
In particular, for counting the number of triangles, we can see that it is enough to count the number of edges in the subgraph induced on the neighbourhood of the vertices. Thus, Local $1-$WL can give the exact count of the number of triangles. For more details, please see \ref{counting_triangle}.
The running time of $1-$WL depends on the number of iterations, $h$. In general, it takes $O((n+m)\log n)$ time, where $m$ is the number of edges, whereas when we look at it in terms of iteration number it requires $O(nh)$ time.

\begin{restatable}{lemma}{lsix}
    It requires $O(n)$ time to guarantee the count of patterns which can be written using $2-$variable with a counting quantifier where the depth of the quantifier is constant.
\end{restatable}
For more details, please see the equivalence between the number of iterations and quantifier depth in \ref{Iteration_number_logic}. A list of patterns that can be counted as subgraph and induced subgraphs using local $1-$WL and local $2-$WL are mentioned in \ref{Corollary Table}. Also, the patterns including 3-star, triangle, chordal $4-$cycle (chordal $C_4$), and attributed triangle have been studied in \citep{chen2020can} and have been shown that it cannot be counted by $1-$WL.
\begin{table}[ht]
\captionsetup{font={small}} 
\caption{\label{Corollary Table}List of all the patterns that can be counted exactly (as a subgraph or induced subgraph), given $G$, using Local $k$-WL, for different $k$.}
\vspace{11pt}
\resizebox{\textwidth}{!}{
\begin{tabular}{llllll}
\toprule
\textbf{Restriction on $G_v^r$} & \textbf{k} & \textbf{Patterns, $H$} & \textbf{Induced subgraph} & \textbf{Subgraph} & \textbf{Reference} \\
\midrule
$G_v^r$ is amenable & 1 & All & \checkmark & \checkmark & Corollary \ref{Corollary 1} \\
Max degree $\leq 5$ & 1 & Patterns with a dominating vertex & \checkmark & \checkmark & Corollary \ref{Corollary 2} \\
Max degree $\leq 15$ & 2 & Pattern with a dominating vertex & \checkmark & \checkmark & Corollary \ref{Corollary 2} \\
No restriction & 2 & 3-star, triangle, tailed triangle, chordal cycle, attributed triangle & \checkmark & \checkmark & Corollary \ref{Corollary 3} \\
No restriction & 1 & Either $H$ or $H-v$ is $K_{1,s}$ or $2K_2,$ where $v$ is the dominating vertex &  & \checkmark & Corollary \ref{1-wl_subgraph} \\
No restriction & 1 & $C_4$ &  & \checkmark & Corollary \ref{c_4 subgraph count} \\
No restriction & 1 & 3-star, tailed triangle, chordal cycle &  & \checkmark & Corollary \ref{neural_comparison_wl} \\
No restriction & 1 & triangle, attributed triangle & \checkmark & \checkmark & Corollary \ref{neural_comparison_wl} \\
\bottomrule
\end{tabular}
}
\end{table}

\subsection{Algorithms for subgraph counting}
\subsubsection{Triangle counting in the host graph}\label{counting_triangle}

We describe an algorithm for counting the number of triangles in a given host graph $G$ in \ref{alg:triangle_count}. Note that counting the number of triangles is the same as counting the number of edges in the subgraph induced by $ N_G(v).$ It is well known that two $1-$WL equivalent graphs have the same number of edges. This ensures that if we run $1-$WL on the induced subgraphs in the neighbourhood of $v$, taking colour as a feature, we can guarantee the count of the triangles. On the other hand, we can see that running $1-$ WL on graph $G$ will not guarantee the count of triangles. Running $1-$WL on the entire graph takes $O(n+m)\log (n)$ and $O(n)$ space, where $m$ is the number of edges. Thus, running $1-$WL locally in the neighbourhood is more space and time-efficient. Note that the running time is also dependent on the number of iterations, $h.$ Running $1-$WL for $h-$ iteration requires $O(nh)$ time. The quantifier depth of counting logic with $(k+1)$ variables is equivalent to the number of iterations of $k-$WL ( See \ref{Iteration_number_logic}). For the case of triangle counting, we just need to count the number of edges, which can be done by running just one iteration of $1-$WL. So, the time required is $O(deg(v))$ for each $v.$ This can be done in parallel for each vertex.
\begin{algorithm}[ht]
\caption{Counting the number of triangles}
\label{alg:triangle_count}
\begin{algorithmic}[1]
\State Let $G$ be the host graph.
\State $num\_edges=0$ 
\For{each vertex $v$ in $V(G)$} \label{algtri:line2}
        \State Find the induced subgraph on $N_G(v)$
        \State Find the number of edges in the induced subgraph on $N_G(v)$
        \State Add it to $num\_edges$
    \EndFor
\State $c=\frac{num\_edges}{3}$
\State \textbf{Output}: The number of triangles in graph $G$ is $c$.
\end{algorithmic}
\end{algorithm}

\subsubsection{Counting subgraph of radius one}
\label{radius_1}

We begin by explaining the procedure for counting the number of subgraphs having a dominating vertex (radius one). For this purpose, we fix a dominating vertex $u$. If a subgraph exists, then the dominating vertex must be mapped to some vertex. We iteratively map the dominating vertex to each vertex in the host graph and count the number of patterns in the neighbourhood of the dominating vertex.

Here, we present an algorithm for counting patterns of radius one \ref{alg:count_radius 1}. Note that running $k-$WL on the entire graph takes $O(k^2\cdot n^{k+1} \log n)$ time and $O(n^k)$ space, whereas when we run locally, it requires less time and space. Suppose we run only on the neighbourhood of each vertex. Then, it requires $\sum_{v\in V(G)} (deg(v))^{k+1}\log(deg(v))$ and space $O(max_i(deg(v_i))^{k}+n).$ More specifically, suppose the given graph is $r-$regular. Then it requires $O(r^{k+1}\log(r) n)$ time and $O(r^k +n)$ space. Therefore, if the graph is sparse, then we can implement Local $k-$WL for a larger value of $k$. We can see that running $k-$WL locally is not dependent on the size of the graph exponentially. However, it is not feasible to run $k-$WL on the entire graph for a larger value of $k.$

\begin{algorithm}[ht]
\caption{Counting the number of patterns of radius one}
\label{alg:count_radius 1}
\begin{algorithmic}[1]
\State Let $H$ be a pattern having a dominating vertex and $G$ be the host graph.

\For{each vertex $v$ in $V(G)$} \label{algrad1:line2}
        \State Find the induced subgraph on $N_G(v)$
        \If{degree$(v)+1<|V(H)|$}
            \State{skip this iteration}
        \EndIf
        \State run $k-$WL on the induced subgraph on $N_G(v)$
        \State Calculate $Count_{u,v}(H,G_v^r)$ 
    \EndFor
    \State \Return $\frac{\sum_{v\in V(G)}Count_{u,v}(H,G_v^r)}{|Orbit_H(u)|} $
\end{algorithmic}
\end{algorithm}

\subsubsection{Counting subgraphs of radius r}

Here, in \ref{alg:count}, we describe how to count the number of subgraphs of radius $r.$ We iterate over all the vertices and take the $r-$hop neighbourhood around that vertex, say $v,$ and choose a suitable $k$ according to the structure of the pattern that can guarantee the count of subgraphs in the local substructure. 
\label{radius_r}
\begin{algorithm}[ht]
\caption{Counting the number of subgraphs}
\label{alg:count}
\begin{algorithmic}[1]
\State Let $P$ be a pattern having a key vertex $u \in V(P)$ and $G$ be the host graph.

\For{each vertex $v$ in $V(G)$} \label{algradr:line2}
    \State Run BFS on $G$ rooted at $v$
    \State $c=1$
    \For{each layer of the BFS tree rooted at $v$}
            \If{$\#$ vertices at layer $i$ from $v$ $<$ $\#$ vertices
           at layer $i$ from $u$}
                    \State{$c=0$}
                    \State skip this iteration
            
            \Else
                \State Color each vertex with color $i$ for layer $i$
                \State This gives a colored graph at each layer
            \EndIf
     \EndFor
     \If{$c==0$}
        \State Skip this iteration
        \State{$c=1$}
    \EndIf
    \State Find an induced subgraph on the set of vertices in the $r-$hop neighborhood
    \State Run $k-$WL on the induced subgraph
    \State Calculate $Count_{u,v}(H,G_v^r)$ 
    \EndFor
    \State \Return $\frac{\sum_{v\in V(G)}Count_{u,v}(H,G_v^r)}{|Orbit_H(u)|} $
\end{algorithmic}

\end{algorithm}

\section{Methodology}
\label{sec:model}
This section describes our unified subgraph counting strategy, combining localized Weisfeiler-Leman (WL) algorithms, differentiable learning, and a novel fragmentation approach. The framework is broken into three primary components: \textbf{(A) Pattern Learning}, \textbf{(B) Local Count Learning}, and \textbf{(C) Global Count Learning}, as illustrated in Figure~\ref{fig:model}.

\begin{figure}[ht]
    \centering
    \includegraphics[width=\textwidth, height=125pt]{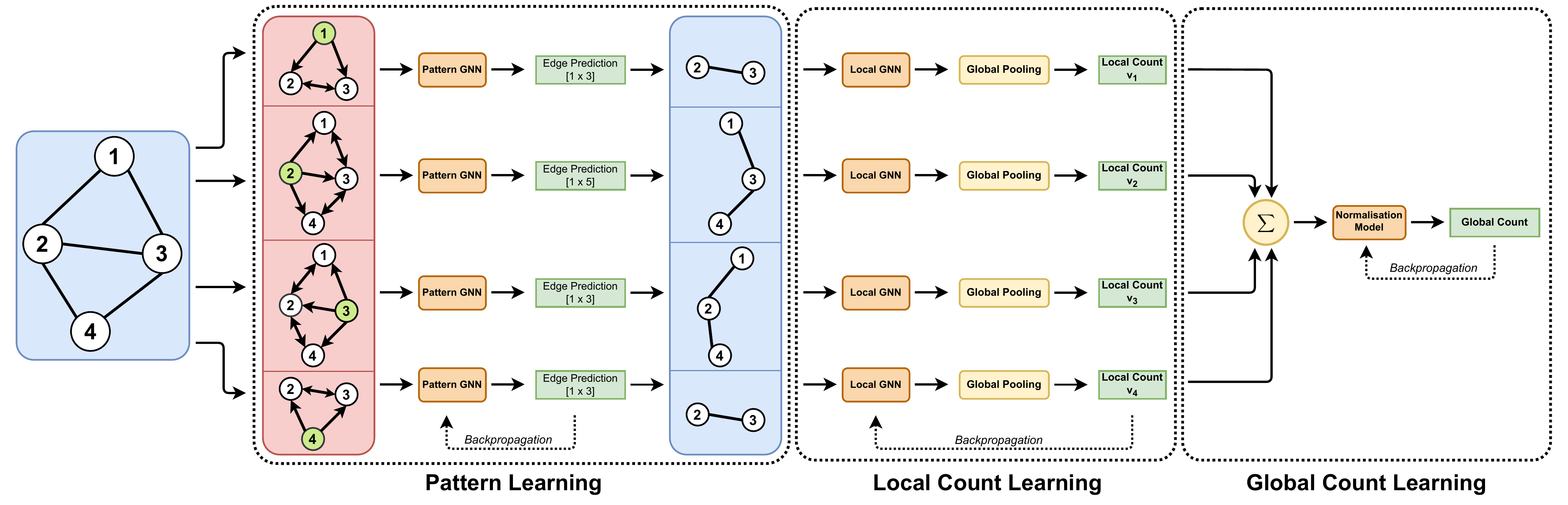}
    \caption{\fontsize{9}{9}\selectfont{Overview of the framework. 
    \textbf{(A) Pattern Learning:} Each node's $r$-hop neighbourhood, with directed edges marking the root, is input to a model that predicts relevant subpatterns. 
    \textbf{(B) Local Count Learning:} Using predicted patterns, the model estimates subpattern counts per root node.
    \textbf{(C) Global Count Learning:} The total count is normalized to adjust for possible overcounting.}}
    \label{fig:model}
\end{figure}

\subsection{Pattern Learning}

For a given subgraph pattern $P$ (e.g., triangle, $3$-star, or tailed triangle), the local subgraph $G_v^r$ centered at node $v$ must be preprocessed to match $P$'s structure. We augment $G_v^r$ by directing edges outward from $v$ to clearly indicate the root. Modified $G_v^r$ is passed to a GNN, which produces node and edge embeddings. These embeddings are used to classify edges as included/excluded based on their relevance to the target pattern, yielding a pruned subgraph suited as input to the next component.

\subsection{Local Count Learning}

Each root node's updated local subgraph (from Pattern Learning) is used to estimate the "local count"—the number of occurrences of a subpattern in $G_v^r$. The local counting rule depends on the motif: for triangles, count edges among neighbours; for $k$-stars, compute $\binom{d_v}{k}$ where $d_v$ is node $v$'s degree. For arbitrary patterns, ground truth local counts are generated as supervision. The model is trained to predict these counts using a GNN applied to the modified $G_v^r$.

\subsection{Global Count Learning}

Summing local counts across all roots may overcount global subgraph occurrences due to overlaps. Hence, a final normalization step is applied. We use a learnable normalization function to infer the correction factor needed to yield the true global count, enabling an end-to-end differentiable pipeline for subgraph counting.

\subsection{Fragmentation}
\label{fragmentation_sec}

For complex patterns beyond primitives like triangles or $k$-stars, we introduce the \emph{fragmentation} technique. Instead of directly counting $P$, which may be computationally intensive, we recursively decompose $P$ into simpler subpatterns ("fragments") whose counts can be efficiently determined. The total count of $P$ is obtained by suitably combining the counts of these fragments, often requiring only small values of $k$ in the local $k$-WL hierarchy.

For example, the tailed triangle can be fragmented into a pendant vertex (the key/root) and an attributed triangle, as shown in Figure~\ref{fig:tail_tri_fragmentation}. For chordal $C_4$ cycles, the pattern is broken into $2$-stars and triangles within the neighbourhood of a key vertex, adjusting counts to avoid overcounting $K_4$s. The approach applies for any pattern whose complexity can be reduced through such decomposition.

\begin{figure}[ht]
    \centering
    \includegraphics[width=0.5375\linewidth]{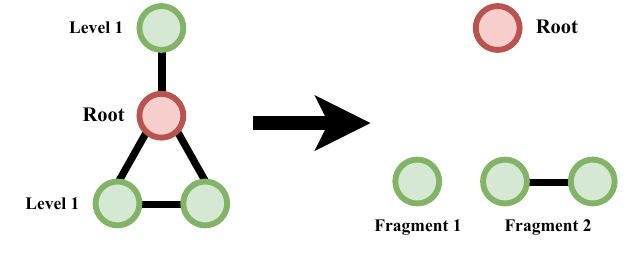}
    \caption{\fontsize{9}{9}\selectfont{Fragmentation example for a tailed triangle: identifying fragments (isolated node, edge) in the subgraph surrounding a root node.}}
    \label{fig:tail_tri_fragmentation}
\end{figure}

\paragraph{Fragmentation Algorithm:}Given a host graph $G$ and a pattern $P$ to be counted, select a key vertex in $P$ and set the radius $r$. For each node $v \in V(G)$, extract the $r$-hop neighbourhood $G_v^r$. Decompose $P$ into subpatterns $\{P_1, ..., P_l\}$, each counted by a learned model $M_i$. Store local counts in arrays, aggregate via learnable combination functions $\alpha$, $\beta$ (for subpattern and root aggregation, respectively), and finally apply a normalization function $\gamma$ to compute the global count.

\begin{algorithm}[ht]
\caption{Fragmentation Algorithm}
\label{alg:Fragmentation}
\begin{algorithmic}[1]
\Require $G$; $ \mathbb{P}$: List of subpatterns; $M^{\text{pattern}}_i$: learned model for generating pattern corresponding to $P_i\in\mathbb{P}$; $M^{\text{count}}_i$: learned model for counting $P_i\in\mathbb{P}$;
\State $a \leftarrow []$
\For{each node $v \in V(G)$}
        \State $H_v = G_v^r \setminus \{v\}$
        \State $b \leftarrow []$
        \For{each pattern $P_i \in \mathbb{P}$}
            \State $c \leftarrow []$
            \For{each node $u\in H_v$}
                \State $\mathcal{P}_u = M^{\text{pattern}}_i(H_v, u)$
                \State $c.append(M^{\text{count}}_i(\mathcal{P}_u))$
            \EndFor
            \State $b.append(\alpha(c))\qquad \#$Learnable function 
        \EndFor
        \State $a.append(\beta(b))\qquad \#$Learnable function 
\EndFor
\State $Count = \gamma(a)\qquad \#$Learnable function 
\State \Return $Count$
\end{algorithmic}
\end{algorithm}
\begin{theorem}
    \label{fragmentation-comparison}
    Using the fragmentation method, we can count induced tailed triangle, chordal $C_4$ and $3-star,$ by running Local $1-$WL.
\end{theorem}
\begin{proof}
    For tailed triangle, we fix the pendant vertex as the key vertex (refer to \ref{fig:tail_tri_fragmentation}). Now, we have to look for the count of triangles such that exactly one of the node of triangle is adjacent to the key vertex. We find induced subgraph of $2-$hop neighborhood for each key vertex. Now, we color the vertices at distance $i$ by color $i$. Then, the problem of counting the number of tailed triangles reduces to counting the number of colored triangles in the colored graph such that one node of triangle is colored $1$ and remaining two nodes are colored $2$. We can find the count of colored triangles using $1-$WL on the induced subgraph by Corollary 1. This number of count of colored triangle is same as $IndCount_{(u,v}(tailed-triangle, G_v^r)$. Now, using \ref{local_induced}, we can say that fragmentation technique can count tailed triangles appearing as induced subgraphs using $1-$WL.

    Consider the pattern, chordal $C_4.$ We have proved that $1-$WL can count the number of subgraphs of chordal $C_4.$ So, to count the number of induced subgraphs of chordal $C_4,$ we only have to eliminate the count of $K_4.$ When we fix one vertex of degree $3$ as key vertex, we can easily compute the count of $K_{1,2}$ in the neighborhood. Now, we have to eliminate all three tuples appearing as triangles in the neighborhood of the key vertex. We can easily count the number of triangles in the neighborhood of each vertex. This gives the exact count of chordal $C_4$ appearing as subgraph in the local structure. Using \ref{local_induced}, we can find $IndCount(Chordal C_4, G)$.

    Consider the pattern, $3-star$. Here, we choose any of the pendant vertex as key vertex. Now, we have to compute $K_{1,2},$ where one center vertex of the star is connected to the key vertex. We can easily count the number of colored $K_{1,2}$ in $2-$hop neighborhood of the key vertex. However, a triangle can be also be included in this count. So, we have to eliminate the $3$ tuples forming a triangle. Again, using the approach discussed above, we can count the number of colored triangles and this will output the exact count of colored induced $K_{1,2}$. Again, using lemma 3, we can find $IndCount(3-star, G).$
\end{proof}
\begin{longtable}{@{}lcl@{}}
\caption{Fragmentation for patterns for size of at most 4} \\
\toprule
\textbf{Vertices} & \textbf{Structure} & \textbf{Formula} \\ 
\midrule
\endfirsthead

\multicolumn{3}{c}%
{{\bfseries Table \thetable\ (continued)}} \\
\toprule
\textbf{Vertices} & \textbf{Structure} & \textbf{Formula} \\ 
\midrule
\endhead

\midrule
\multicolumn{3}{r}{{Continued on next page}} \\
\bottomrule
\endfoot

\bottomrule
\endlastfoot

\multirow{-5}{*}{G1} & \includegraphics[scale=0.75]{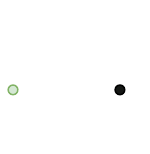} & \multirow{-5}{*}{$\binom{n}{2}-|E|$} \\
\midrule
\multirow{-5}{*}{G2} & \includegraphics[scale=0.75]{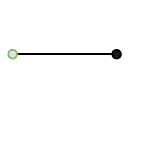} & \multirow{-5}{*}{$|E|$} \\ \midrule
\multirow{-5}{*}{G3} & \includegraphics[scale=0.75]{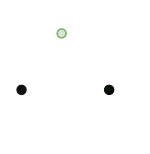} & \multirow{-5}{*}{$\frac{\sum_v IndCount_{(u,v)}(G_1, G-G[N_G[v]])}{3}$} \\ \midrule
\multirow{-5}{*}{G4} & \includegraphics[scale=0.75]{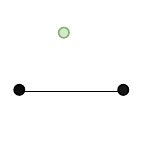} & \multirow{-5}{*}{$\sum_v IndCount(G_2, G-G[N_G[v]])$} \\ \midrule
\multirow{-5}{*}{G5} & \includegraphics[scale=0.75]{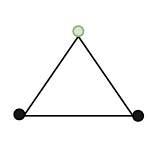} & \multirow{-5}{*}{$\frac{\sum_v IndCount_{(u,v)}(G_2, G[N_G(v)])}{3}$} \\ \midrule
\multirow{-5}{*}{G6} & \includegraphics[scale=0.75]{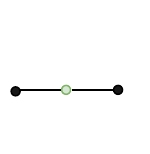} & \multirow{-5}{*}{$\sum_v \binom{degree(v)}{2}-|E(N_G(v))|$} \\ \midrule
\multirow{-5}{*}{G7} & \includegraphics[scale=0.75]{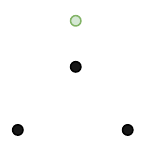} & \multirow{-5}{*}{$\frac{\sum_v IndCount_{(u,v)}(G_3, G-G[N_G(v)]}{4}$} \\ \midrule
\multirow{-5}{*}{G8} & \includegraphics[scale=0.75]{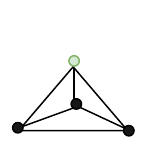} & \multirow{-5}{*}{$\frac{\sum_v IndCount_{(u,v)}(G_5, G[N_G[v]])}{4}$} \\ \midrule
\multirow{-5}{*}{G9} & \includegraphics[scale=0.75]{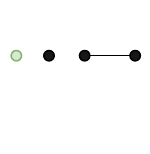} & \multirow{-5}{*}{$\frac{\sum_v IndCount_{(u,v)}(G_4, G- G[N_G[v]])}{2}$} \\ \midrule
\multirow{-5}{*}{G10} & \includegraphics[scale=0.75]{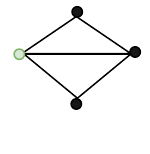} & \multirow{-5}{*}{$\frac{\sum_v (IndCount(G_6,N_G(V))- IndCount(G_5,N_G(v)))}{2}$} \\ \midrule
\multirow{-5}{*}{G11} & \includegraphics[scale=0.75]{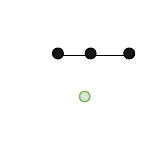} & \multirow{-5}{*}{$\sum_v (IndCount(G_6,G-G[N_G[v]])$} \\ \midrule
\multirow{-5}{*}{G12} & \includegraphics[scale=0.75]{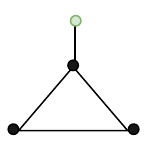} & \multirow{-5}{*}{$\sum_v Count(attribute*-triangle, G[N_G(v)\cup N_G(N_G(v))])$} \\ \midrule
\multirow{-5}{*}{G13} & \includegraphics[scale=0.75]{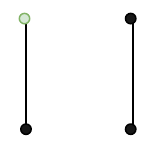} & \multirow{-5}{*}{$\binom{|E(G)|}{2} -Count(k_{1,2},G)$} \\ \midrule
\multirow{-5}{*}{G14} & \includegraphics[scale=0.75]{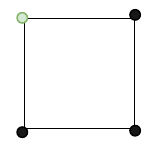} & \multirow{-5}{*}{$\frac{\sum_v IndCount(attributed-k_{1,2},G[N_G(v)\cup N_G(N_G(v))])}{4}$} \\ \midrule
\multirow{-5}{*}{G15} & \includegraphics[scale=0.75]{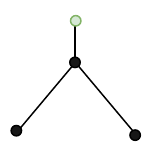} & \multirow{-5}{*}{$\sum_v IndCount(attribute-k_{1,2},G[N_G(v)\cup N_G(N_G(v))])$} \\ \midrule
\multirow{-5}{*}{G16} & \includegraphics[scale=0.75]{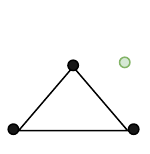} & \multirow{-5}{*}{$\frac{\sum_v (n-deg(v)-1)\cdot |E(G[N_g(v)]|}{3}$} \\ \midrule
\multirow{-7}{*}{G17} & \includegraphics[scale=0.75]{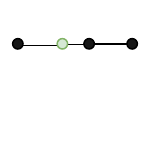} & \multirow{-7}{*}{$\frac{IndCount(attributed-G_4, G[N_G(v)\cup N_G(N_G(v))]}{2}$} \\
\end{longtable}
\label{tab:fragmentation_upto_4}
\begin{theorem}
    \label{4-sized graph}
    Using the fragmentation technique we can count all patterns appearing as induced subgraphs of size $4$ by just running Local $1-$WL.
\end{theorem}
\begin{proof}
In \ref{tab:fragmentation_upto_4}, we describe how the fragmentation technique can be leveraged to count all the induced subgraphs of size $4$. This shows that for $k=1$, the fragmentation technique is more expressive than $S_{k+3}$. Also, we can enlist more graphs where the larger pattern can be seen as the union of fragments of smaller patterns. Using this, we can see that it covers all the graphs that were mentioned in \citep{chen2020can}. One can see that all the formulae use the function that can be computed by $1-$WL. The number of vertices and the number of edges can easily be computed after doing one iteration of $1-$WL. Also, the degree sequence can be computed after running $1-$WL for one iteration. All other functions(formulae) are just the linear combination of the functions computed earlier, the number of vertices or the number of edges.
\end{proof}
\paragraph{Expressiveness Guarantee:}Using fragmentation, we show that all patterns of up to $4$ nodes can be exactly counted using only Local $1$-WL. For specific motifs (tailed triangle, chordal $C_4$, $3$-star), the decomposition ensures algorithmic tractability and high expressiveness (see Theorems~\ref{fragmentation-comparison} and~\ref{4-sized graph}).

The fragmentation technique thus enables scalable, modular, and highly expressive subgraph counting, leveraging learnable models for each subcomponent while ensuring practical feasibility and end-to-end differentiable training.

\section{Theoretical Comparison of Graph Neural Networks}\label{sec:compare_mk_local}
n this section we give a comparison of the time and expressiveness between the GNN hierarchies proposed in \citep{papp2022theory} and our methods. From \ref{local_wl_power}, it is clear that $k-$WL is less expressive than Local $k-$WL. Also, we have shown that the space and time required by Local $k-$WL are less compared to $k-$WL. 

$S_k$ is a model in which a vertex is given an attribute based on the number of connected induced subgraphs of size at most $k,$ the key vertex. Even though it searches locally, the number of non-isomorphic graphs may be too many for a small value of $k$. Suppose the radius of the induced subgraph is $r$; then it has to search up to the $r-$hop neighborhood. Using brute force would require $O(n_1^k)-$time to count the number of induced subgraphs of size $k,$ for each individual induced subgraph. To improve the time complexity, it either stores the previous computation, which requires a lot of space or further recomputes from scratch. Thus, it would require $O(t_k\times n_1^k)$, where $ t_{k}$ is the number of distinct graphs upto isomorphism of at most $k-$vertices. Using \ref{induced_subgraph_general}, one can easily see that running $k-$WL locally is more expressive than $S_{k+2}.$

The $N_k$ model has a preprocessing step in which it takes the $k-$hop neighborhood around vertex $v,$ and gives attributes based on the isomorphism type. In a dense graph, it is not feasible to solve the isomorphism problem in general, as the size of the induced subgraph may be some function of $n$. Also, using the best-known algorithm for graph isomorphism by \citet{babai2016graph}, the time required is $O(n_1^{O(\log n_1)}).$ 
However, running Local $k-$WL would require $O(n_1^k)$. Also, there are rare examples of graphs that are $3-$WL equivalent and non-isomorphic. So, if we run $3-$WL locally, then most of times expressive power matches with $N_k.$

The $M_k$ model deletes a vertex $v$ and then runs $1-WL.$ \citet{papp2022theory} proposed that instead of deleting the vertices, $k$ vertices are marked in the local neighborhood and showed that it is more expressive than deletion. It identifies $k$ set of vertices in the local $r-$hop neighborhood of the graph. It would require $O(n_1^{(k+2)}\log (n_1))$ time as it has $O(n_1^k)$ many possibilities of choosing $k$ many vertices. It requires $O(n^2 \log n)$ time to run $1-$WL. The same time is required for Local $(k+1)-$WL. 
Also, it is known that with any $l$ markings and running $k-$WL is less expressive than running $(k+l)-$WL on the graph \citep{furer2017combinatorial}. So, if we plug in the value, we can see that running Local $k-$WL is more expressive than doing $l$ marking and running $1-$WL. One can get an intuition by comparing with the $(k+1)$ bijective pebble game. If we mark the vertices, then the marking pebbles get fixed and give less power the to spoiler. However, just running $k-$WL the spoiler is free to move all the pebbles. We present a simple proof that Local $k-$WL is at least as expressive as $M_{k-1}$.
\begin{theorem}
\label{$M_{k-1}$}
Local $k-$WL is atleast as expressive as $M _{k-1}$. 
\end{theorem}
\begin{proof}
Let $G_{v}^r$ and $G_{u}^r$ be the induced subgraphs around the $r-$hop neighborhood for vertices $v$ and $u$, respectively. Let $M_{k}$ distinguish $G_{v}^r$ and $G_{u}^r$. We claim that Local $k-$WL can also distinguish the graph $G_{v}^r$ and $G_{u}^r$. To prove our claim, we use the pebble bijective game. $G_{v}^r$ is distinguished because there exists a tuple $(v_{1},v_{2},....v_{k-1})$ such that marking these vertices and running $1-$WL on the graph gives stabilised coloring of the vertices that does not match with that of $G_{u}^r.$ Now, consider two games. One game corresponds to $1-$WL and another to Local $k-$WL. For the first $(k-1)$ moves, the Spoiler chooses to play and places pebbles at $(v_{1},v_{2},....v_{k-1}).$ After that, in both games, there are two pebbles and the position of both games are the same. Let $S_1$ and $D_1$ be the spoiler and Duplicator in the $(k+1)$ bijective pebble game, and $S_2$ and $D_2$ be the spoiler and Duplicator in $2$ bijective pebble game. $S_1$ will follow the strategy of $S_2$ and $D_2$ follows the strategy of $D_1.$ We prove by induction on the number of rounds. Our induction hypothesis is that the position of games in both the games is the same and if $S_2$ wins, then $S_1$ also wins. 

\textit{Base case :} Duplicator $D_1$ proposes a bijection. $D_2$ will propose the same bijection. Now, $S_2$ places a pebble on some vertex $v.$ Similarly, $S_1$ will also place a pebble at $v.$ Note that the position of the game is the same, and if $S_2$ wins, then $S_1$ also wins. 

Now, using the induction hypothesis, assume that the position of both the games is the same and $S_2$ has not won till round $i.$ 

Now, consider the game at round $(i+1).$ If $S_2$ decides to play / remove a pebble, then $S_1$ will do the same. If $S_2$ decides to play a pebble, then $S_1$ also decides to play a pebble. So, $D_1$ proposes a bijective function. $D_2$ proposes the same bijective function. Now, $S_2$ places pebble at some vertex $u,$ then $S_1$ also places pebble at $u$. Thus, the position of both the game is the same and if $S_2$ wins, then $S_1$ will also win. 
\end{proof}

\section{Experiments}
\label{sec:experiments}
\subsection{Implementation Details}
We refer to our proposed model as \textbf{InSig}. In the experiments, we predict the counts of the following  substructures occurring as subgraphs :  triangles, $3$-Star, $2$-Star, chordal $C_4$, $K_4$, $C_4$, and tailed triangles. 
\begin{table}[ht]
\centering
\resizebox{0.9\textwidth}{!}{%
\begin{tabular}{@{}lrrrrrr@{}}
\toprule
\textbf{Task} &
  \textbf{\begin{tabular}[c]{@{}r@{}}Total pattern\\ count\end{tabular}} &
  \textbf{\begin{tabular}[c]{@{}r@{}}Zero Count\\ graphs\end{tabular}} &
  \textbf{\begin{tabular}[c]{@{}r@{}}Standard Deviation\\ of Count\end{tabular}} &
  \textbf{\begin{tabular}[c]{@{}r@{}}Average number\\ of Nodes\end{tabular}} &
  \textbf{\begin{tabular}[c]{@{}r@{}}Average number\\ of Edges\end{tabular}} &
  \textbf{\begin{tabular}[c]{@{}r@{}}Number\\ of graphs\end{tabular}} \\ \midrule
\textit{Triangle}        & $25209$  & $195$  & $3.072$  & \multirow{7}{*}{$18.7976$} & \multirow{7}{*}{$62.678$} & \multirow{7}{*}{$5000$} \\
\textit{2-Star}          & $429463$ & $0$    & $18.015$ &                          &                         &                       \\
\textit{3-Star}          & $309525$ & $0$    & $17.777$ &                          &                         &                       \\
\textit{Chordal}         & $19088$  & $1786$ & $4.742$  &                          &                         &                       \\
\textit{K4}              & $643$    & $4447$ & $0.387$  &                          &                         &                       \\
\textit{C4}              & $53002$  & $16$   & $6.938$  &                          &                         &                       \\
\textit{Tailed Triangle} & $177968$ & $195$  & $25.943$ &                          &                         &                       \\ \bottomrule
\end{tabular}%
}
\caption{\fontsize{9}{9}\selectfont{Dataset statistics. The total number of graphs in the dataset is $5000.$ We used $4000$ graphs for training, $500$ for validation and $500$ for testing. }}
\label{tab:dataset}
\end{table}
When counting larger substructures like $K_4$, $C_4$, and Tailed Triangles, we use the \textbf{fragmentation} technique with the help of a model learned to predict triangles, 3-Star, 2-Star, and Chordal $C_4$. For counting $K_4$, the task of \textbf{Pattern Counting} Component is to learn which edges to prune. Later \textbf{Local Count Learning} component counts the number of substructures,  which in the case of $K_4$, is triangles. Once the number of triangles is predicted using models learned to count triangles, the normalization factor is learned to output the global count. In other structures like tailed triangles, we have two patterns to learn: the nodes in the $1$-hop neighborhood and the edges between them. During the inference phase, we use a rounding function, as the counts are integer numbers. 
\subsection{Hyperparameters}
We use two GIN Convolutional layers for the Pattern Learning Local Count Learning component. We consider Linear transformations as readout layers in the previously mentioned components. Since we need to classify whether an edge should be present or not in a subgraph for counting a subpattern, we use Binary Cross-Entropy (BCE) loss for the pattern learning component. For the local counting and global counting components, we use Mean Absolute Error (MAE). For the models we have considered, as there can be paths of lengths more than 1 in the subgraph, $2$ \emph{GINConv} layers are sufficient to capture the information well. 

We use a learning rate of $1e-4$ and a batch size of $1.$ We also experimented with different hidden dimensions for the node embeddings and obtained the best results when we used $4$ as a hidden dimension size. The experiments were conducted using an NVIDIA A100 40GB GPU. The source code of the implementation is available at this \href{https://github.com/Roy-Shubhajit/InSig-GNN}{\faGithub\ Github Link}.
\subsection{Experimental Results}

\begin{table}[ht]
\centering
\resizebox{0.9\textwidth}{!}{%
\begin{tabular}{@{}crrrrrrr@{}}
\toprule
\multirow{2}{*}{\textbf{Models}} &
  \multicolumn{4}{c}{\textbf{Without Fragmentation}} &
  \multicolumn{3}{c}{\textbf{Fragmentation}} \\ \cmidrule(l){2-5}  \cmidrule(l){6-8}
 &
  \textit{Triangle} &
  \textit{3-Star} &
  \textit{2-Star} &
  \textit{Chorcal C4} &
  \textit{$K_4$} &
  \textit{$C_4$} &
  \textit{Tailed Triangle} \\ \midrule
\multicolumn{1}{l}{\textbf{ID-GNN}} & 6.00E-04 & NA       & NA & 4.52E-02 & 2.60E-03 & 2.20E-03 & 1.05E-01 \\
\multicolumn{1}{l}{\textbf{NGNN}}   & 3.00E-04 & NA       & NA & 3.92E-02 & 4.50E-03 & 1.30E-03 & 1.04E-01 \\
\multicolumn{1}{l}{\textbf{GNNAK+}} & 4.00E-04 & 1.50E-02 & NA & 1.12E-02 & 4.90E-03 & 4.10E-03 & 4.30E-03 \\
\multicolumn{1}{l}{\textbf{PPGN}}   & 3.00E-04 & NA       & NA & 1.50E-03 & 1.65E-01 & 9.00E-04 & 2.60E-03 \\
\multicolumn{1}{l}{\textbf{I2-GNN}} & 4.00E-04 & NA       & NA & 1.00E-03 & 3.00E-04 & 1.60E-03 & 1.10E-03 \\
\multicolumn{1}{l}{\textbf{InSig}} &
  \textbf{0.00E-00} &
  \textbf{0.00E-00} &
  \textbf{0.00E-00} &
  \textbf{0.00E-00} &
  \textbf{0.00E-00} &
  \textbf{0.00E-00} &
  \textbf{0.00E-00} \\ \bottomrule
\end{tabular}%
}
\caption{\fontsize{9}{9}\selectfont{MAE for the subgraph count of different patterns. Some results, such as \emph{$2-$star}, \emph{$3-$star}, are not conducted by the given baselines; therefore, it is mentioned NA.}}
\label{tab:results}
\end{table}
The dataset used for the experiments is a random graphs dataset prepared in \citep{chen2020can}. In Table \ref{tab:dataset}, we report the dataset statistics, specifically the counts of the various patterns used for our experiments, number of graphs where these patterns do not appear, and so on. We report our experiments' Mean Absolute Error (MAE) in Table \ref{tab:results}. We compare our results with those reported in \emph{ID-GNN} \citep{https://doi.org/10.48550/arxiv.2101.10320}, \emph{NGNN} \citep{zhang2021nested}, \emph{GNNAK+} \citep{zhao2022from}, \emph{PPGN} \citep{https://doi.org/10.48550/arxiv.1905.11136} and \emph{I2-GNN} \citep{huang2023boosting}.
We can observe that our approach of predicting substructures, the local counts, and then predicting the global counts, achieves zero error for all of the tests.

For all the patterns, we observed that the model gets to zero error after only 2 to 3 epochs. From Table \ref{tab:param}, it can be observed that our approach requires a considerably smaller number of parameters and beats all the baselines.
Our inference time comprises model inference as well as the graph preprocessing time, where it creates the set of subgraphs corresponding to each node in the graph.

\begin{table}[ht]
\centering
\resizebox{0.5\textwidth}{!}{%
\begin{tabular}{lrrr}
\toprule
\textbf{Models} &
  \textbf{\begin{tabular}[c]{@{}r@{}}Number of\\ Parameters\end{tabular}} &
  \textbf{\begin{tabular}[c]{@{}r@{}}Inference\\ Time (ms)\end{tabular}} &
  \textbf{\begin{tabular}[c]{@{}r@{}}Memory\\ Usage (GB)\end{tabular}} \\ \midrule
\textbf{ID-GNN} & 102K         & 5.73        & 2.35         \\
\textbf{NGNN}   & 127K         & 6.03        & 2.34         \\
\textbf{GNNAK+} & 251K         & 16.07       & 2.35         \\
\textbf{PPGN}   & 96K          & 35.33       & 2.3          \\
\textbf{I2-GNN} & 143K         & 20.62       & 3.59         \\
\textbf{InSig}  & \textbf{268} & \textbf{20} & \textbf{1.2} \\ \bottomrule
\end{tabular}%
}
\caption{\fontsize{9}{9}\selectfont{The table shows the comparison of the number of parameters required by the baselines and \emph{InSig Model}. The values shown here correspond to the triangle counting task with the hidden dimension set as $4$.\label{tab:param}}}
\end{table}
\begin{figure}[ht]
    \centering
    \begin{subfigure}{0.45\textwidth}
        \centering
        \includegraphics[width=\linewidth]{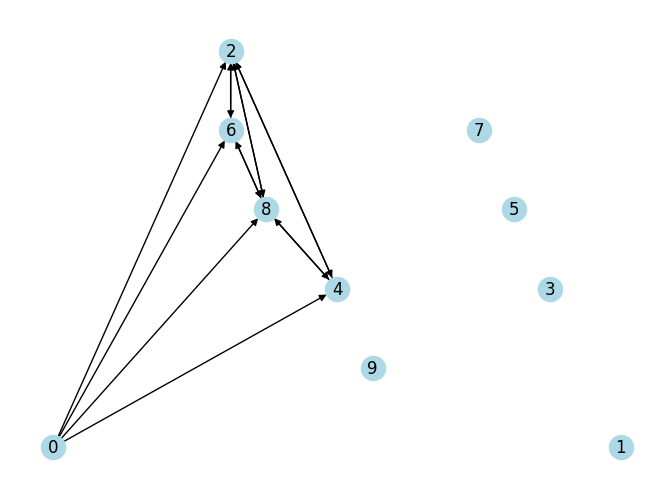}
        \caption{\fontsize{9}{9}\selectfont{Figure shows an example of $G^1_0$ for a random graph with $10$ nodes}}
        \label{fig:sub1}
    \end{subfigure}%
    \hfill
    \begin{subfigure}{0.45\textwidth}
        \centering
        \includegraphics[width=\linewidth]{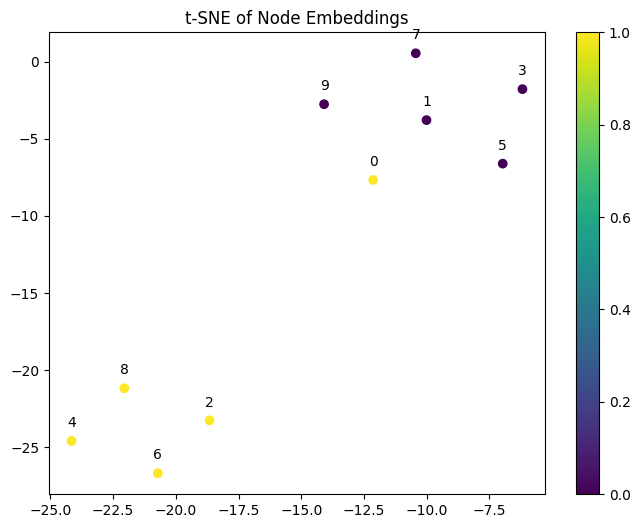}
        \caption{\fontsize{9}{9}\selectfont{t-SNE plot of the updated node embeddings from the Pattern learning component}}
        \label{fig:sub2}
    \end{subfigure}
    \caption{\fontsize{9}{9}\selectfont{Figure showing the input and updated node embedding from the pattern learning component}}
    \label{fig:both}
\end{figure}
In Figure \ref{fig:both}, we show an example of a graph sent as input to the pattern learning component and the updated node embeddings outputted from the model. In Figure \ref{fig:sub1}, we see that the root node is node $0$ and there are directed edges from the root node to its neighboring nodes. Figure \ref{fig:sub2} shows the updated node embeddings of each node in the graph. We can observe that $8, 2, 4, 6$ nodes can be separated using some separator from the rest of the nodes. This indicates that the model is able to learn the updated node embedding such that we can distinguish nodes which has edges between them. 

\section{Conclusion}\label{conclusion}
Subgraph counting is a fundamental combinatorial problem that arises in the study of graphs and graph structured problems as well as data. Exactly counting the number of subgraphs in a graph is  also a computationally hard problem. In this paper,  we present a learnable algorithm that is able to compute exact counts of a number of commonly occurring patterns. The proposed fragmentation method has proven to be  beneficial for the task of counting subgraphs. Additionally, since fragmentation results in smaller subgraphs, the parameter requirements of the proposed GNN-based architecture is orders of magnitude less than previous methods. As future work, we plan to analyse the fragmentation algorithm from a theoretical perspective. We also plan to investigate the effectiveness of this technique for increasing the expressiveness of GNNs for downstream tasks. Along with that, the current approach is limited to patterns that can be computed using $1$-hop or $2-$hop subgraphs, hence we can study our algorithm for bigger patterns.

\bibliographystyle{plainnat}
\bibliography{reference}




\newpage
\end{document}